\documentclass{article} 
\usepackage{arxiv}

\usepackage{amsmath,amsfonts,bm}

\def\eqref#1{equation~\ref{#1}}

\def\1{\bm{1}}

\DeclareMathAlphabet{\mathsfit}{\encodingdefault}{\sfdefault}{m}{sl}
\SetMathAlphabet{\mathsfit}{bold}{\encodingdefault}{\sfdefault}{bx}{n}

\usepackage[utf8]{inputenc} 
\usepackage[T1]{fontenc}    
\usepackage{hyperref}       
\usepackage{url}            
\usepackage{booktabs}       
\usepackage{amsfonts}       
\usepackage{nicefrac}       
\usepackage{microtype}      
\usepackage{xcolor}         
\usepackage{times} 
\usepackage{amsmath}
\usepackage{amsthm}
\usepackage{natbib}
\usepackage{amssymb}
\usepackage{graphicx}
\usepackage{wrapfig}
\usepackage{booktabs}
\usepackage{multirow}
\usepackage{float}
\usepackage{placeins}
\usepackage{natbib} 

\newtheorem{theorem}{Theorem}

\newtheorem{proposition}{Proposition}
\newtheorem{corollary}{Corollary}
\newtheorem{assumption}{Assumption}
\newtheorem{definition}{Definition}
\newtheorem{remark}{Remark}

\theoremstyle{remark}

\DeclareMathOperator{\hsic}{\text{HSIC}}

\title{Principled Nonlinear Concept Erasure and Steering}

\author{Idan Simai \\
  Department of Computer Science \\
  Bar-Ilan University \\
  \texttt{idansi98@gmail.com} \And
  Shauli Ravfogel \\
  Center for Data Science \\
  New York University \\
  \texttt{shauli.ravfogel@nyu.edu} \And
  Uri Shaham \\
  Department of Computer Science \\
  Bar-Ilan University \\
  \texttt{uri.shaham@biu.ac.il}}

\date{}                                  
\renewcommand{\shorttitle}{Principled Nonlinear Concept Erasure and Steering}  

\begin{document}

\maketitle

\begin{abstract}
Learned representations often encode sensitive attributes, such as gender or race, that may be undesirable or unsafe to retain.
Concept erasure aims to remove this information while preserving task-relevant content. 
Existing methods remain incomplete: linear methods provide guarantees only against linear probes, while nonlinear methods achieve stronger empirical erasure but lack formal guarantees and often sacrifice utility.
In this work, we introduce a principled conditional autoencoder framework for nonlinear concept erasure. We follow the principle that erasure should remove only the sensitive attribute and nothing else, and we show that, under the stated assumptions, perfect reconstruction together with population-level independence is \emph{sufficient} for perfect nonlinear guardedness, with preservation guarantees in the imperfect-reconstruction regime.
Empirically, we attain a superior utility-erasure trade-off across multiple datasets, sensitive attributes, and language models, surpassing even the supervised variants of competing methods despite using no target-task labels, while being simpler to optimize and converging substantially faster.
Beyond erasure, our framework yields \emph{concept conversion (representation steering)} as a by-product: by changing only the target concept at inference time, it produces representations corresponding to different concept values---at no additional training cost---a capability no prior erasure method provides, and on which we likewise demonstrate strong empirical performance.
\end{abstract}

\section{Introduction}
Pretrained Language Models (PLMs) have become the backbone of modern Natural Language 
Processing (NLP), achieving strong performance across a wide range of tasks. However,
their representations often encode sensitive attributes---such as gender, race, or other 
demographic factors---which can lead to biased and potentially unfair predictions
\cite{bolukbasi2016man, deArteaga2019bias, ravfogel2020null}.
\emph{Concept erasure} addresses this issue by modifying the learned representations to remove 
such undesirable concepts, while retaining as much task-relevant information as possible 
\cite{basu-roy-chowdhury-etal-2021-adversarial, ravfogel2022linear, yamada2022leace}. The task may be performed either in a supervised manner, leveraging task labels 
to guide utility preservation, or in an unsupervised manner, where no task labels 
are assumed.

Existing concept erasure methods can be broadly divided into two families. Linear methods are well-studied~\cite{bolukbasi2016man, ravfogel2020null, ravfogel2022linear}, with some coming with formal theoretical guarantees of removing  linearly-accessible information~\cite{yamada2022leace}. However, by construction, all such methods protect only against \emph{linear} adversaries. A nonlinear probe can readily recover the sensitive attribute from the erased representation, preventing \emph{exhaustive} removal~\cite{ravfogel-etal-2022-adversarial}. Nonlinear methods attempt to close this gap~\cite{elazar2018adversarial, ravfogel-etal-2022-adversarial, chowdhury2022learning, basu2023robust, shao2023gold, saillenfest2025nonlinearconcepterasuredensity, jourdan2023taco, akbari2025obliviator}, and while they achieve stronger empirical erasure, existing approaches offer limited guarantees of complete nonlinear concept removal, leaving uncertainty as to whether the attribute is removed or merely harder to recover. This leaves practitioners with an uncomfortable choice--provable but incomplete linear erasure, or empirically stronger but theoretically ungrounded nonlinear erasure.

In this work, we propose an autoencoder model for nonlinear concept erasure
that comes equipped with theoretical guarantees under stated population-level
conditions. Our approach is motivated by a simple 
principle: the goal of concept erasure is to remove \emph{only} the sensitive 
attribute--and nothing else. This is naturally enforced by maximizing the mutual 
information between the erased representation $Z$ and the original input $X$, subject 
to $Z$ being statistically independent of the sensitive attribute $C$.
Under perfect reconstruction and
independence, we prove that the learned representation is an entropy-preserving 
function of the complement of the sensitive attribute, and further establish guarantees 
for the near-perfect regime.

Figure~\ref{fig:bios_bert} presents a representative headline result on \textsc{Bias in Bios}~\cite{deArteaga2019bias}: our fully unsupervised method achieves a strictly better utility--erasure trade-off than all competing approaches, even outperforming supervised Obliviator~\cite{akbari2025obliviator}, the current state-of-the-art. Beyond these significant empirical improvements, we show
that these two properties are \emph{sufficient} conditions for perfect nonlinear
guardedness in autoencoder-based models. Furthermore, our framework is
significantly easier to train and converges faster than prior nonlinear methods.

In addition to erasure, our framework yields a bonus capability: \emph{concept conversion (representation steering)}.  Since the decoder takes both the erased representation $Z$ and the sensitive attribute $C$ as input, 
substituting a different value $C'$ at inference time generates a representation 
corresponding to a different concept value--with no additional training required.
\begin{figure}[!t]
    \centering
    \includegraphics[width=0.65\textwidth]{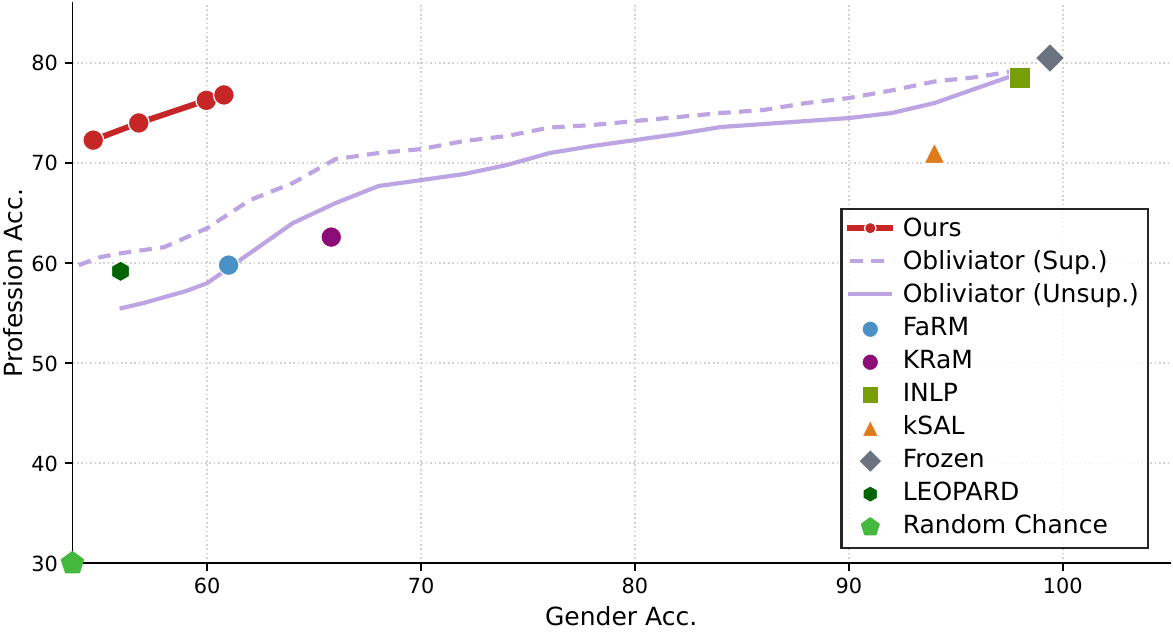}
    \vspace{-0.75em}
\caption{
Strong concept erasure result on BERT representations for profession prediction with gender as the sensitive concept.
The plot trades off gender leakage ($x$-axis; lower is better) against profession utility ($y$-axis; higher is better).
Red points correspond to different independence strengths in our objective.
Our method achieves a substantially better utility--erasure trade-off than competing baselines, including both supervised and unsupervised Obliviator, while being unsupervised. 
}
\label{fig:bios_bert}
    \label{fig:bert_bias_in_bios}
\end{figure}

Our contributions are fourfold: (i) \textbf{Theoretical guarantees.} We prove that the two properties targeted by our objective, are \emph{sufficient}
conditions for perfect nonlinear guardedness in autoencoder-based models.
We further establish preservation guarantees in the imperfect-reconstruction
regime, quantifying how much complementary (non-sensitive) information is
retained as a function of the reconstruction error.
(ii) \textbf{Concept conversion.} Our trained decoder enables conversion (steering) across concept values at no additional training cost: by simply substituting a different sensitive attribute value at inference time, we can generate a representation as if the input had a different concept value.
(iii) \textbf{New state-of-the-art.} Without target-task labels in our objective, our model achieves superior utility-erasure trade-offs compared to  existing methods, across multiple datasets, sensitive 
attributes, and language models, outperforming even the supervised variants of competing methods.
(iv) \textbf{Simpler and faster optimization.} Our two-term objective is straightforward to optimize, significantly simpler than prior nonlinear methods, and converges faster.

\section{Related Work}

\paragraph{Linearly Guarded Erasure.}
Early concept erasure methods are linear: they identify directions in representation space associated with the sensitive attribute and project them out~\cite{bolukbasi2016man,ravfogel2020null}. INLP~\cite{ravfogel2020null} repeatedly projects representations onto the nullspace of linear concept classifiers, R-LACE~\cite{ravfogel2022linear} formulates erasure as a minimax game against a linear adversary, and LEACE~\cite{yamada2022leace} gives a closed-form least-squares projection with provable guarantees for linear concept removal. However, these methods guard only against linear adversaries: nonlinear probes can still recover the sensitive attribute from the erased representation. Moreover, none supports concept steering.

\paragraph{Nonlinearly Guarded Erasure.}
Several works propose nonlinear erasure methods to overcome the limits of linear approaches. Adversarial methods \cite{elazar2018adversarial, 
ravfogel-etal-2022-adversarial} train nonlinear adversaries to suppress the sensitive attribute, but erasure is often incomplete and optimization can be unstable.  Rate-distortion based methods such 
as FaRM \cite{chowdhury2022learning} and its kernelized variant KRaM \cite{basu2023robust} offer stronger empirical erasure, but rely on specialized rate-distortion objectives. 
Other approaches include density matching~\cite{saillenfest2025nonlinearconcepterasuredensity}, which aligns distributions across sensitive groups, and TaCo~\cite{jourdan2023taco}, which removes concept-specific latent directions. Obliviator 
\cite{akbari2025obliviator}, the current state-of-the-art, directly targets nonlinear 
statistical dependencies via RKHS and achieves strong empirical performance. While these 
nonlinear methods advance the empirical state of the art, they provide limited guarantees of complete nonlinear erasure, and none supports concept steering. Our 
method addresses both gaps simultaneously.

\paragraph{Independence-Constrained Representation Learning.}
Our method builds on prior work that learns representations which preserve useful information while suppressing dependence on a specified variable. Theoretically, \cite{shaham2022discovery} show that an autoencoder with an independence-promoting discriminator can recover a hidden independent component up to an entropy-preserving transformation. 
A related latent-independence perspective appears in speech attribute conversion~\cite{svirsky2025provable}; here we study its implications for concept erasure and guarded representations. 
In practice, attribute-invariant encoder--decoders have been used for controllable generation and style transfer in images~\cite{lample2017fader}, voice and music~\cite{qian2019autovc,nachmani2019unsupervised,mor2019universal}, and text~\cite{john2019disentangled}. Similar ideas appear in fairness-oriented representation learning~\cite{louizos2016variational,edwards2016censoring,ganin2016domain}, often adversarially; however, adversarial invariance alone does not guarantee concept removal, as post-hoc classifiers may still recover the attribute~\cite{elazar2018adversarial}. Thus, our contribution is not a new independence mechanism, but a concept-erasure analysis showing that reconstruction plus independence yields nonlinear guardedness guarantees and, as a by-product, steering.

\paragraph{Representation-Level Steering.}
Representation-level steering modifies frozen representations to change a concept while preserving other content. Earlier work studied feature or data repair for fairness, reducing protected-attribute information or aligning group-conditional distributions~\cite{feldman2015certifying,johndrow2019algorithm,gordaliza2019obtaining}. 
More recently, Representation Surgery~\cite{singh2024representation} formalized affine maps such as mean matching and mean--covariance matching, while related counterfactual methods use similar moment-matching interventions for explanation and bias mitigation~\cite{lemberger2024explaining}. 
These approaches provide affine or distributional transformations, but remain limited to population-statistic-based interventions and do not guarantee exact counterfactual recovery or nonlinear guardedness. 
In contrast, our conditional encoder--decoder yields nonlinear steering by decoding the concept-independent representation with a target concept value, requiring no additional objective and remaining tied to our erasure guarantees.

\section{Problem Formulation}
\label{sec:problem_formulation}
Let $C \in \mathcal{C}$ be an observed random variable representing the \emph{sensitive concept} that we wish to erase, where
$\mathcal{C}$ is a discrete set, and let
$U \in \mathcal{U}$ denote an \emph{unobserved complementary random variable} containing
all non-sensitive information that should be preserved. Let $X \in \mathcal{X}$ be the observed
representation. We assume that $X$ is generated by an unknown deterministic map
\(
X = f(U, C),
\)
where $f : \mathcal{U} \times \mathcal{C} \to \mathcal{X}$ is injective. We further assume that $U$ and $C$ are independent.

\begin{remark}
The assumption $U \perp\!\!\!\perp C$ is required for our theoretical guarantees and is standard in ICA-style latent-variable models, but it is nevertheless idealized. If $U$ is statistically entangled with $C$, then part of the information in $U$ may itself reveal $C$, making perfect preservation of $U$ and erasure of $C$ no longer simultaneously possible. To understand what happens when independence holds only approximately, Appendix~\ref{app:robustness} presents a controlled simulation in which the dependence between $U$ and $C$ is varied continuously. We find that small deviations from independence lead to correspondingly small degradation in both erasure and steering, and the framework continues to deliver strong empirical performance well outside the perfect regime.

\end{remark}

Our setting can thus be viewed as a special case of nonlinear ICA in which one factor, $C$, is observed,
while the other factor, $U$, is latent and must be preserved.
In the context of concept erasure, $C$ is the information that should be removed from the learned
representation, whereas $U$ is exactly the complementary content that should remain. Let $\mathcal{P}_{UC}=\mathcal{P}_{U}\times \mathcal{P}_{C}$ be a distribution on $\mathcal{U} \times \mathcal{C}$ and let $\{(u_i, c_i)\}_{i=1}^n$ be $n$ samples from it,
and for all $i = 1, \ldots, n$, let $x_i = f(u_i, c_i)$.
Given only input-concept pairs, i.e., $\{(x_i, c_i)\}_{i=1}^n$, the corresponding
complementary factors $\{u_i\}_{i=1}^n$
are never observed. No task labels are assumed.

\paragraph{Concept erasure.}
Our goal is to learn an erased representation that removes information about the sensitive concept
$C$ while preserving the complementary information carried by $U$. To this end, we learn a
conditional encoder--decoder map
\(
g(x,c) = d(e(x),c),
\)
where
\(
e : \mathcal{X} \to \mathcal{Z},d : \mathcal{Z} \times \mathcal{C} \to \mathcal{X}
\)
are an encoder and decoder, respectively, and
\(
Z := e(X)
\)
is the erased representation.

\paragraph{Concept conversion.}
We are also interested in the induced \emph{concept conversion (representation steering)} task. Given an input $x = f(u,c)$ and a target concept value $c' \in \mathcal{C}$, the goal is to generate a converted sample corresponding to the same complementary factor $u$ under the new concept value $c'$, namely
\(
x' = f(u,c').
\)
In our framework, this is realized by
\(
g(x,c') = d(e(x),c').
\)
Thus, erasure and steering are studied here within the same conditional encoder--decoder framework.

\section{Proposed Method}
\subsection{Motivation}
\label{sec:motivation}

The guiding principle of unsupervised concept erasure is to remove \emph{only} the sensitive
concept and preserve everything else. In the notation of Section~\ref{sec:problem_formulation},
this suggests the ideal population-level objective
\begin{equation}
\max_{e} I(X;Z)
\qquad
\text{subject to}
\qquad
Z=e(X), \;\; Z \perp\!\!\!\perp C .
\label{eq:ideal_mi_objective}
\end{equation}
That is, among all representations that are statistically independent of the sensitive concept \(C\),
we seek the one that retains as much information as possible about the original input \(X\).

A key observation is that, under the independence constraint $Z \perp\!\!\!\perp C$,
maximizing $I(X;Z)$ is equivalent to minimizing the conditional uncertainty
$H(X \mid Z, C)$; the formal derivation is given in Appendix~\ref{app:mi_vs_hsic}.
Therefore, reconstruction from $(Z, C)$ is not merely a heuristic surrogate for
information preservation, but the correct operational objective. In particular, given a decoder $d$ with reconstruction
\(
\hat{X} = d(Z, C),
\)
driving the reconstruction error to zero forces $H(X \mid Z, C)$ to vanish,
and thus maximizes the amount of information about $X$ retained in the erased
representation among all $C$-independent representations.

This is the main reason we adopt an autoencoder formulation. The decoder plays a principled
role: it certifies that the erased representation still contains all non-sensitive information needed
to recover the original input once the concept value \(C\) is supplied. This same structure also
naturally enables steering, since the decoder can be queried with a new concept value
\(c'\).

This perspective also clarifies a limitation of HSIC-based preservation objectives such as the
one used in Obliviator. While HSIC is a useful dependence
measure, maximizing \(\mathrm{HSIC}(X,Z)\) is not generally equivalent to maximizing
\(I(X;Z)\). Suitably normalized or spectrally regularized formulations can mitigate this
issue and, under additional structure, recover the same ideal information-preserving
optimum. Obliviator employs normalization, but this does not in general enforce the
covariance whitening or spectral isotropy required for that equivalence. Thus, its
preservation objective remains distinct from mutual-information maximization. The two
measures agree at the minimum---for characteristic kernels, both vanish iff
\(X \perp\!\!\!\perp Z\)---but not in general at the maximum: mutual information is
invariant under invertible reparameterizations of \(Z\), whereas HSIC depends on the
chosen kernel geometry and bandwidth. As a result, HSIC maximization may favor
representations aligned with the chosen kernel rather than those preserving the largest
amount of information. Appendix~\ref{app:mi_vs_hsic} formalizes this distinction and
gives a simple counterexample.

For this reason, our model is built around two ingredients: reconstruction, which targets
information preservation, and statistical independence from the sensitive concept, which targets
erasure.

\subsection{Method}
\label{sec:method}
In our setting, \(X\) is a representation produced by a pretrained language model, \(C\) is the concept label to be erased (e.g., gender, race), and \(Z\) is the resulting erased representation.
We train an autoencoder, consisting of an encoder
\(
e:\mathcal{X}\to\mathcal{Z}
\)
and a decoder
\(
d:\mathcal{Z}\times\mathcal{C}\to\mathcal{X}
\).
Given an input \(x\) and its concept value \(c\), the encoder produces an erased representation
\(
z=e(x)
\),
and the decoder reconstructs the input as
\(
\hat{x}=d(z,c)=d(e(x),c).
\)

Our training objective combines two terms:
\begin{equation}
\mathcal{L}
=
\mathcal{L}_{\mathrm{rec}}
+
\lambda_{\mathrm{ind}} \mathcal{L}_{\mathrm{ind}},
\label{eq:objective}
\end{equation}
where \(\lambda_{\mathrm{ind}}>0\) controls the trade-off between reconstruction quality and concept removal.

\paragraph{Reconstruction.}
To preserve the complementary information in the input, we reconstruct \(x\) from the erased representation \(z\) together with the concept \(c\). We implement the reconstruction term using mean squared error:
\(
\mathcal{L}_{\mathrm{rec}}
=
\frac{1}{n}\sum_{i=1}^n
\left\|
x_i - d(e(x_i),c_i)
\right\|_2^2.
\)
\paragraph{Independence.}
To remove sensitive-concept information, we encourage statistical independence
between $Z = e(X)$ and $C$. In practice, we implement this term using an
HSIC-based dependence penalty: $L_{\mathrm{ind}} = \mathrm{HSIC}(Z,C)$.
Minimizing HSIC is appropriate here because, with characteristic kernels,
its population value is zero if and only if $Z$ and $C$ are independent,
which also implies $I(Z;C)=0$. Unlike in the preservation objective, we do
not rely on HSIC as a surrogate for mutual information.

Thus, the overall objective is simply built around the two ingredients required for unsupervised concept erasure: reconstruction, which preserves the non-sensitive information needed to recover the input once \(C\) is given, and independence, which removes information about the sensitive concept from the latent representation. At inference time, the erased representation is \(z=e(x)\), and steering is obtained by decoding the same representation with a target concept value \(c'\):
\(
g(x,c') = d(e(x),c').
\)

\section{Theoretical Analysis}

We now present the theoretical guarantees underlying our framework.

\paragraph{Scope of the guarantees.}

Our guarantees characterize the properties of representations that attain the
conditions targeted by our objective, rather than convergence of the finite,
non-convex training procedure itself. In the ideal regime, perfect
reconstruction together with population-level independence, $Z \perp C$,
is sufficient for perfect nonlinear guardedness and, as a by-product, yields
exact steering through the same decoder at no additional training cost.
Beyond this idealization, we quantify complementary-information preservation
directly as a function of the reconstruction error actually attained. In
practice, Eq.~(\ref{eq:objective}) optimizes finite-sample estimates and is not
guaranteed to reach the ideal population conditions. Nevertheless, across
datasets and language models, the learned representations achieve
state-of-the-art empirical erasure and steering performance, supported
by held-out dependence tests and nonlinear probes.

\paragraph{Discrete setting.}
We state the entropy-based results formally in the discrete setting, where the
central argument can be expressed cleanly using Shannon entropy. A fully general
continuous treatment would require additional measure-theoretic machinery that
would obscure the simple mechanism we aim to expose without changing its core
intuition. This theoretical restriction does not apply in practice: our experiments operate on continuous neural representations, and the
reconstruction-error guarantee in Section~\ref{sec:imperfect} is stated
directly in Euclidean error and does not rely on discreteness.

\subsection{The Perfect Regime}
\label{sec:perfect_regime}

We begin by isolating the two structural properties that underlie our ideal-case analysis and formalizing the target guarantee.

\textbf{Perfect reconstruction:} the decoder recovers the original input when conditioned on the true concept value, i.e.\ $d(Z, C) = X$ a.s., where $C$ is the concept variable that generated $X = f(U, C)$.

\textbf{Perfect independence:} the learned representation $Z = e(X)$ is statistically independent of the sensitive concept, $Z \perp\!\!\!\perp C$.

\begin{definition}[Perfect nonlinear guardedness]
\label{def:perfect_guardedness}
An autoencoder \(g(x,c)=d(e(x),c)\) is \emph{perfectly nonlinearly guarded} with respect to \(C\) if, letting \(Z=e(X)\),
\(
I(C;Z)=0,
\)
and there does not exist a deterministic representation \(Z' = e'(X)\) such that
\(
I(C;Z')=0, 
I(X;Z') > I(X;Z).
\)
Equivalently, \(Z\) is statistically independent of the sensitive concept \(C\), and is maximally informative about \(X\) among all \(C\)-independent deterministic representations of \(X\).
\end{definition}

Our first result shows that if the decoder reconstructs the input exactly and the erased representation is independent of the sensitive concept, then the representation preserves all the non-sensitive information in the input.
\begin{theorem}[Maximal information retention]
\label{thm:perfect_preservation}
Assume that
\(
\hat X = d(Z,C) = X \text{ a.s,}
\)
\(
Z \perp\!\!\!\perp C.
\)
Then:
\begin{equation}
I(X;Z) = H(X \mid C) = H(U).
\label{eq:perfect_ixz}
\end{equation}
\end{theorem}

\begin{corollary}[Optimality among \(C\)-independent deterministic representations]
\label{cor:optimality_c_independent}
Under the assumptions of Theorem~\ref{thm:perfect_preservation}, for every deterministic representation
\(
Z' = e'(X)
\)
such that \(I(C;Z')=0\), we have
\(
I(X;Z') \le I(X;Z).
\)
\end{corollary}

\begin{corollary} [Sufficient conditions for perfect nonlinear guardedness]
\label{cor:perfect_guardedness}
Under the assumptions of Theorem~\ref{thm:perfect_preservation}, the autoencoder \(g\) is perfectly nonlinearly guarded in the sense of Definition~\ref{def:perfect_guardedness}.
\end{corollary}

Three further consequences of Theorem~\ref{thm:perfect_preservation},
namely deterministic recovery of the complementary factor
\(H(U \mid Z, C) = 0\), exact preservation of the complementary information
\(I(U; Z \mid C) = H(U)\), and the entropy-preserving relation
\(H(Z) = H(U)\), together with the proofs of Theorem~\ref{thm:perfect_preservation}
and all corollaries above, are deferred to Appendix~\ref{app:perfect_regime_proofs}.

\subsection{The Imperfect Regime}
\label{sec:imperfect}
We now turn to the non-ideal regime. Since autoencoders are typically trained with mean-squared
reconstruction error, we seek guarantees stated directly in terms of MSE. The exact statement above
suggests a natural relaxation: if the reconstruction from \((Z,C)\) is nearly perfect, then \(Z\) should
preserve nearly all information about \(U\).

\paragraph{Why an additional assumption is needed.}
A small MSE in the observation space does not necessarily imply small error in the recovered
complementary factor \(U\). MSE controls distances in \(X\)-space, whereas preservation of \(U\)
depends on the stability of the inverse map from \((X,C)\) back to \(U\). To translate reconstruction error in \(X\)-space into
recovery error for \(U\), we impose the following regularity assumption on this inverse map.

\begin{assumption}[Lipschitz left-inverse]
\label{ass:lipschitz_left_inverse}
There exists a measurable map $h:\mathcal X \times \mathcal C \to \mathcal U$
such that
\(
h(f(u,c),c)=u,
\forall (u,c)\in \mathcal U \times \mathcal C,
\)
and for some constant \(L_h \ge 0\),
\(
\|h(x,c)-h(x',c)\|_2 \le L_h \|x-x'\|_2,
\forall x,x'\in \mathcal X,\ \forall c\in \mathcal C.
\)
\end{assumption}
Assumption~\ref{ass:lipschitz_left_inverse} states that, once \(c\) is fixed, the complementary
factor \(U\) can be recovered from \(x\) in a stable way. This is natural for a conditional
autoencoder: reconstruction from \((Z,C)\) is useful only if small errors in \(X\)-space do not
correspond to large errors in the latent factor \(U\). The Lipschitz condition rules out such
ill-conditioned inverse mappings from \((X,C)\) to \(U\).

All expectations in this subsection are taken with respect to the joint distribution induced by
\(
(U,C)\sim \mathcal{P}_{UC}, X=f(U,C), Z=e(X).
\)
Proposition~\ref{prop:mse_implies_preservation} shows that small reconstruction error in \(X\)-space
implies a small recovery error for the complementary factor \(U\). 
In this sense, the encoding preserves
the non-sensitive information up to a controlled distortion.
\begin{proposition}[Small reconstruction error implies approximate preservation]
\label{prop:mse_implies_preservation}
Assume that $\mathbb E\!\left[\|X-d(Z,C)\|_2^2\right] \le \varepsilon_{\mathrm{rec}}$,
and that Assumption~\ref{ass:lipschitz_left_inverse} holds. 
Define $\hat U := h(d(Z,C),C)$.
Then
\(
\mathbb E\!\left[\|U-\hat U\|_2^2\right] \le L_h^2\,\varepsilon_{\mathrm{rec}}.
\)

\end{proposition}

\paragraph{From MSE to information preservation.}
Proposition~\ref{prop:mse_implies_preservation} translates into an
information-theoretic guarantee whenever the alphabet $\mathcal{U}$
of the complementary factor $U$ is finite with separated support.
Informally, under these mild conditions, a small reconstruction error implies that the encoder loses almost no information about
the complementary factor. 
The full statement and proof are
deferred to Appendix~\ref{app:imperfect_regime}
(Corollary~\ref{cor:finite_alphabet_guardedness}), as well as the proof of Proposition~\ref{prop:mse_implies_preservation}.

\begin{remark}
Corollary~\ref{cor:finite_alphabet_guardedness} does not require any independence condition on $Z$ and $C$. Its conclusion is purely a preservation statement: reconstruction accuracy controls how well the complementary factor $U$ is retained.
\end{remark}

\subsection{Steering as a By-product of Erasure}
Because the decoder is conditioned on the concept value, the same model that achieves concept erasure can also be used for steering by decoding the erased representation with a different concept input. The next result shows that this holds exactly under the same ideal conditions as perfect guardedness, up to the following mild symmetry-breaking assumption on $U$.

\begin{assumption}[Asymmetric law of $U$]
\label{ass:asym_u}
The only measure-preserving automorphism $\pi:\operatorname{supp}(U)\to\operatorname{supp}(U)$
such that $\pi_{\#}P_U = P_U$ is the identity.
\end{assumption}

This is a mild symmetry-breaking assumption, needed only to rule out nontrivial measure-preserving relabelings of $U$. For a discrete $U$, it holds whenever distinct support points have distinct probabilities.

The following corollary shows that, for any $x=f(u,c)$ and $c'$, decoding with $c'$ yields exactly $f(u,c')$, preserving $u$ and changing only the concept. Thus, the same conditions that yield perfect nonlinear guardedness (Corollary~\ref{cor:perfect_guardedness}) also yield exact steering at no additional training cost.

\begin{corollary}[Exact steering under perfect reconstruction and independence]
\label{cor:conversion}
Suppose Assumption~\ref{ass:asym_u} holds, and let $g(x, c) = d(e(x), c)$ satisfy perfect reconstruction, $d(e(X), C) = X$ a.s., and perfect independence, $Z = e(X) \perp\!\!\!\perp C$. Then there exists a bijection $T : \mathrm{supp}(U) \to \mathrm{supp}(Z)$
such that
\(
e(f(u, c)) = T(u) ,\forall u \in \mathrm{supp}(U),\ \forall c \in \mathrm{supp}(C).
\)
Consequently, for every target concept value $c' \in \mathrm{supp}(C)$,
\(
d(e(f(U, C)), c') = f(U, c'), \text{a.s.}
\)
\end{corollary}
The proof is given in Appendix~\ref{app:conversion-proof}.

\section{Experiments}
\label{sec:experiments}
\subsection{Experimental Setup}
\label{subsec:setup}
We evaluate erasure and steering across multiple datasets and pretrained language models, using the same representations for both tasks.

\paragraph{Benchmarks and PLM embeddings.}
We evaluate on three standard concept-erasure benchmarks: \textsc{Bias in Bios}~\cite{deArteaga2019bias} 
(target: profession, $28$ classes; sensitive: gender), \textsc{Dial-Sentiment}~\cite{blodgett2016demographic} 
(target: sentiment; sensitive: race), and \textsc{Dial-Mention}~\cite{blodgett2016demographic} 
(target: mention; sensitive: race). We extract frozen representations from BERT-base~\citep{devlin2019bert}, 
GPT-2~\cite{radford2019language}, DeepSeek-LLM-7B~\cite{bi2024deepseek}, and LLaMA-3.2-1B~\cite{grattafiori2024llama}, 
using the \texttt{[CLS]} embedding for BERT and mean-pooled token embeddings for the remaining models, following~\cite{akbari2025obliviator}.

\paragraph{Baselines.} For erasure, we compare against Obliviator~\cite{akbari2025obliviator} in both its supervised and unsupervised variants, as well as FaRM~\cite{chowdhury2022learning}, KRaM~\cite{basu2023robust}, LEOPARD~\cite{saillenfest2025nonlinearconcepterasuredensity}, INLP~\cite{ravfogel2020null}, and kSAL~\cite{shao2023gold}.
For steering, we compare against Representation Surgery~\cite{singh2024representation}, the most directly related affine-steering framework, including its Mean Matching and Mean+Covariance Matching (MiMiC) interventions, and ActAdd~\citep{turner2024steeringlanguagemodelsactivation}.

\paragraph{Implementation details.}
We keep the main text focused on the empirical findings and defer all low-level implementation details, hyperparameters, and complexity analysis to Appendix~\ref{app:implementation}. 

\paragraph{Evaluation protocol.}
We organize the empirical evaluation around two tasks: concept erasure and
steering. For erasure, we train post-hoc classifiers on the
erased representations: a sensitive-attribute classifier to measure leakage and
a downstream-task classifier to measure utility. We study the resulting utility--erasure trade-off by training separate models with different values of the independence weight \(\lambda_{\mathrm{ind}}\) in Eq.~\ref{eq:objective}, yielding one point per value.
For steering, given an input representation $x$ with concept value
$c$ and a target concept value $c'$, we compute $d(e(x),c')$ and evaluate
whether the converted representation preserves downstream target information
while shifting toward the target concept. We assess these properties using downstream-task classifiers, fairness metrics, and post-hoc sensitive-concept classifiers.

\subsection{Concept Erasure Results}
We evaluate concept erasure with post-hoc classifiers. The $x$-axis reports sensitive-attribute accuracy, where lower is better and random chance is ideal; the $y$-axis reports downstream-task accuracy, where higher is better. Thus, better methods lie closer to the upper-left. Baselines are plotted in their standard form: iterative methods as trade-off curves and one-shot methods as single operating points. For Obliviator, the curve-selection protocol is described in Appendix~\ref{sec:exp_setup}.
\begin{figure}[!h]
    \centering
    \includegraphics[width=1.00\textwidth]{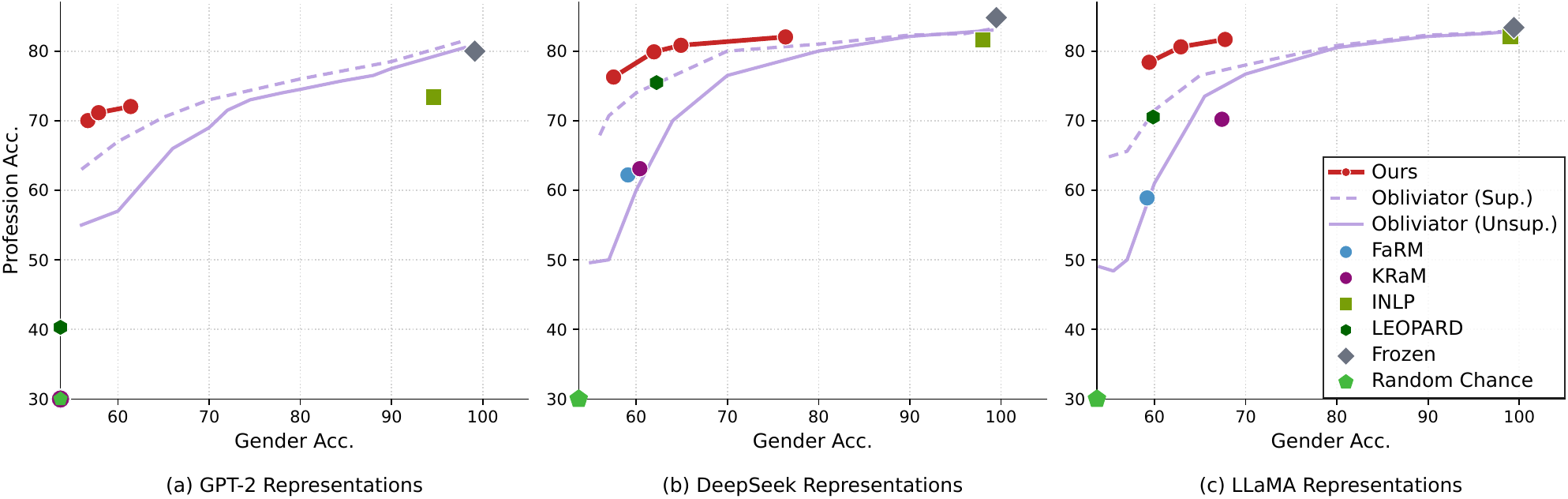}
    \vspace{-0.75em}
    \caption{
Concept erasure on \textsc{Bias in Bios} across frozen PLM representations.
Repeating Figure~\ref{fig:bios_bert} on GPT-2, DeepSeek, and LLaMA, our method achieves a stronger trade-off than competing baselines, including supervised Obliviator.
This shows that the advantage of reconstruction plus independence is not tied to one representation space, but generalizes across PLMs and embedding dimensions.
}
    \label{fig:bios_plms}
\end{figure}

\paragraph{Bias in Bios.}
Figures~\ref{fig:bios_bert} and~\ref{fig:bios_plms} evaluate concept erasure on
\textsc{Bias in Bios}, where the sensitive attribute is gender and the downstream
target is profession. On BERT representations, Figure~\ref{fig:bios_bert} shows
that our method achieves a substantially stronger utility--erasure trade-off. At approximately random-chance gender
accuracy, around $56\%$, our method preserves about $74\%$ profession accuracy,
compared with only about $56\%$ and $61\%$ for unsupervised and supervised
Obliviator, respectively. Thus, at essentially the same erasure level, our method
improves utility by roughly $18$ and $13$ percentage points, despite using no
profession labels during erasure. The remaining baselines either sacrifice
substantially more profession accuracy or retain considerably higher gender
leakage, placing them below our trade-off.\\
Figure~\ref{fig:bios_plms} shows that this advantage persists across other PLM
representations. On GPT-2, at around $58\%$ gender accuracy, our method retains
about $70\%$ profession accuracy, compared with roughly $56\%$--$65\%$ for
Obliviator. On DeepSeek, at around $57\%$ gender accuracy, we retain
$76.27\%$ profession accuracy, compared with roughly $50\%$--$70\%$ for
Obliviator. On LLaMA, at around $63\%$ gender accuracy, the corresponding
profession accuracies are about $81\%$ for our method versus $66\%$--$74\%$ for
Obliviator. These gaps support our central design choice:
using reconstruction as the preservation objective yields stronger utility retention
than HSIC-based preservation, and this advantage is not tied to a particular PLM
representation space.
\begin{figure}[!h]
    \centering
    \includegraphics[width=0.75\textwidth]{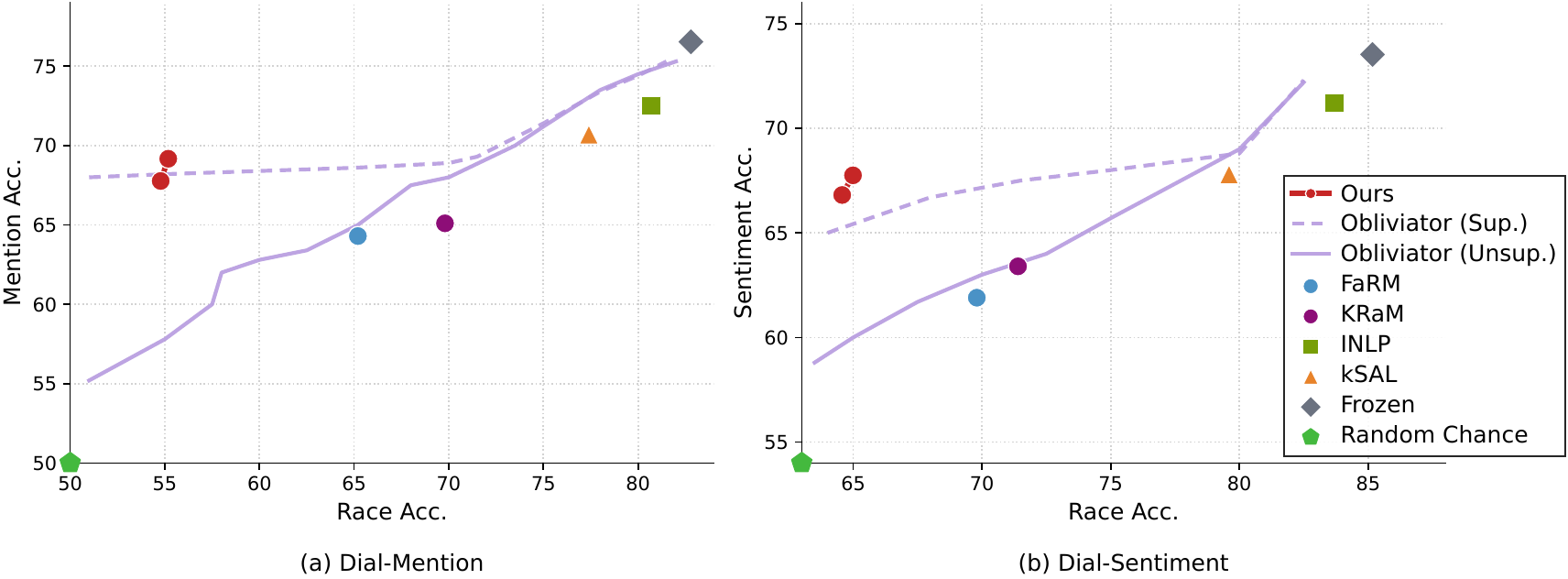}
    \vspace{-0.75em}
\caption{Concept erasure on \textsc{Dial-Mention} and \textsc{Dial-Sentiment} using BERT representations.}
    \label{fig:dial_bert}
     \vspace{-1.0em}
\end{figure}
\paragraph{Dial datasets.}
Figure~\ref{fig:dial_bert} evaluates erasure on Dial. Across both tasks, our
method substantially outperforms unsupervised Obliviator at comparable race
leakage. On \textsc{Dial-Mention}, our unsupervised method matches supervised
Obliviator's utility--erasure trade-off, while on \textsc{Dial-Sentiment} it
achieves higher downstream utility in the low-leakage regime, despite using no
downstream labels. Appendix~\ref{sec:fariness} further shows that our method's advantage extends beyond erasure, achieving a consistently stronger fairness--utility trade-off on \textsc{Dial-Sentiment} under both DP and GAP$_{\mathrm{RMS}}$.
Additional probes and an exact-HSIC permutation test are reported
in Appendices~\ref{app:more_probes} and~\ref{app:permutation}.

\subsection{Generative Steering}
\label{sec:steering}

We evaluate whether the steering capability induced by our erasure model
transfers to generation, using toxicity steering on Qwen2.5-7B~\citep{qwen2.5}
with \textsc{RealToxicityPrompts}~\citep{gehman2020realtoxicityprompts}. The autoencoder is trained on final-layer states supplied to the frozen LM head,
using toxicity labels but no generation-specific objective.
\begin{wraptable}{r}{0.54\columnwidth}
\vspace{-1.7\baselineskip}
\centering
\footnotesize
\setlength{\tabcolsep}{2.5pt}
\caption{Methods are compared
at matched topicality; lower toxicity and PPL are better.}
\label{tab:toxicity_steering}
\begin{tabular}{lccc}
\toprule
\textbf{Intervention} &
\textbf{Tox. $\downarrow$} &
\textbf{PPL $\downarrow$} &
\textbf{Topic.} \\
\midrule
Untouched & 0.163 & 3.44 & 0.387\\
\midrule
\multicolumn{4}{l}{\emph{Topicality $\approx 0.35$}}\\
Mean matching, $\alpha=1.5$ & 0.095 & \textbf{3.74} & 0.358\\
ActAdd, $\alpha=1$        & 0.076 & 4.03 & 0.335\\
MiMiC, $\alpha=0.5$       & 0.078 & 3.94 & 0.351\\
Ours, $\beta=1$           & \textbf{0.074} & 3.81 & 0.352\\
\midrule
\multicolumn{4}{l}{\emph{Topicality $\approx 0.30$}}\\
ActAdd, $\alpha=1.5$      & 0.047 & 4.98 & 0.301\\
MiMiC, $\alpha=0.9$       & \textbf{0.034} & 5.28 & 0.302\\
Ours, $\beta=2$           & 0.037 & \textbf{4.77} & 0.303\\
\midrule
\multicolumn{4}{l}{\emph{Topicality $\approx 0.26$}}\\
ActAdd, $\alpha=2$        & 0.029 & \textbf{6.60} & 0.264\\
MiMiC, $\alpha=1.25$      & \textbf{0.018} & 7.66 & 0.266\\
Ours, $\beta=3$           & \textbf{0.018} & 6.86 & 0.263\\
\bottomrule
\end{tabular}
\vspace{-1.6em}
\end{wraptable}
Our framework naturally converts a representation through
$d(e(h),c')$. During autoregressive generation, however, repeatedly replacing
the complete hidden state can accumulate decoder reconstruction error.
We therefore use
\(
    h'
    =
    h
    +
    \beta
    \left[
        d(e(h),c')
        -
        d(e(h),c)
    \right].
    \label{eq:residual_steering}
\)
This retains the model's original hidden state while applying the nonlinear
concept transformation learned by the decoder. Under perfect reconstruction and $\beta=1$, this reduces to the original
conversion. For detoxification, every state is steered from toxic to non-toxic,
requiring no state-level toxicity prediction at inference time.

Since stronger steering trades detoxification for generation quality, we sweep the strength of every method
and compare operating points at matched topicality. We evaluate continuation
toxicity together with perplexity under the untouched model and
prompt--continuation similarity.

Table~\ref{tab:toxicity_steering} shows a strong detoxification--quality trade-off.
Around topicality $0.35$, ours achieves the lowest toxicity while also improving
fluency over activation addition and MiMiC. Around topicality $0.30$, it remains
close to MiMiC in toxicity ($0.037$ vs.\ $0.034$) while substantially improving
PPL ($4.77$ vs.\ $5.28$). At the strongest operating point, ours matches MiMiC's
toxicity ($0.018$) with markedly lower PPL ($6.86$ vs.\ $7.66$). Thus, the nonlinear steering mechanism learned by our erasure model transfers
effectively to autoregressive generation, without retraining the LM or learning
a separate steering model.
Appendix~\ref{app:steering_setup} gives the full setup and
Appendix~\ref{sec:add_results} the full results, including example
generations, autoregressive gender steering on Llama-2-7B, and frozen-probe
steering across PLMs.

\section{Conclusion}
We introduced a conditional autoencoder framework for concept erasure that combines
reconstruction-based preservation with statistical independence from the sensitive
concept. Under the stated ideal conditions, perfect reconstruction and population
independence are sufficient for nonlinear guardedness and exact steering, while
imperfect reconstruction admits quantitative preservation guarantees. Empirically,
our method achieves stronger utility--erasure trade-offs across datasets and language
models, while also enabling effective nonlinear steering in representation-level and
generative settings. Overall, reconstruction plus independence provides a simple and
practical framework for nonlinear concept erasure and steering.

\clearpage

\section*{AI use statement}
We used generative AI tools to assist with drafting and polishing parts of the manuscript. In particular, they were used to prepare initial versions of some passages and to improve wording, clarity, and organization. All AI-assisted text was subsequently reviewed, edited, and approved by the authors.

We did not use generative AI tools for literature retrieval or discovery, research ideation or execution, generating synthetic datasets, or proving mathematical claims. Technical claims, equations, experimental results, and citations were checked by the authors against our own analysis and the relevant sources.

We take full responsibility for the final content of the paper, including all text and other material.

\section*{Ethics statement}
All real-data experiments use publicly released datasets: BIAS IN BIOS, the dialectal Twitter
corpus behind the DIAL benchmarks, and REALTOXICITYPROMPTS. Appendix \ref{app:robustness} additionally
contains a fully synthetic simulation. We collected no new human-subject data and ran no
studies with human participants. The race labels in the DIAL benchmarks are dialect-based
proxies rather than self-reported attributes, and we use them only as the concept to erase.

Concept erasure can keep sensitive attributes such as gender or race from influencing downstream
decisions, but it is not a fairness guarantee: an attribute can be entangled with task-relevant content,
and too little erasure can give a false sense of protection. The same decoder that removes a concept
can also steer toward one, including toward undesirable behavior. In our toxicity experiments,
we steer only away from toxicity; we additionally study controlled gender steering on BIAS IN BIOS.
The toxicity appendix contains model-generated continuations of offensive prompts; we selected
examples without slurs or explicit content.

\section*{Reproducibility statement}
Proofs of all theoretical results are in Appendix~\ref{app:proofs}, and the
synthetic setting used to test them when independence is violated is in
Appendix~\ref{app:robustness}. The datasets and their splits are described in
Appendix~\ref{app:datasets}. The architecture, objective, HSIC estimator,
evaluation protocol and probes are in Appendix~\ref{sec:exp_setup}, and the
hyperparameters behind the results are in
Tables~\ref{tab:hparams} and~\ref{tab:figure_hyperparams}. The generative
experiments, including the intervention point, decoding and scoring,
are specified in Appendix~\ref{app:steering_setup}, and the metrics are
defined in Appendices~\ref{app:steering_setup} and~\ref{app:metrics}.

\clearpage

\bibliography{paper}
\bibliographystyle{plainnat}

\appendix
\appendix
\clearpage

\section{Proofs}
\label{app:proofs}
\subsection{Why reconstruction targets mutual information and HSIC does not}
\label{app:mi_vs_hsic}

We formalize the motivation behind the objective in
\eqref{eq:ideal_mi_objective}. We show two facts:
(i) under the independence constraint \(Z \perp\!\!\!\perp C\), maximizing \(I(X;Z)\) with respect to $Z$ is
equivalent to minimizing \(H(X \mid Z,C)\), which justifies reconstruction as the preservation
objective; and
(ii) Mutual information is not a monotonic function of HSIC, so maximizing HSIC is not, in
General, the correct way to optimize information preservation.

\begin{proposition}\label{prop_2}
Let \(Z=e(X)\) be a deterministic representation, and suppose that \(Z \perp\!\!\!\perp C\).
Then:
\begin{equation}
I(X;Z) = H(X \mid C) - H(X \mid Z,C).
\label{eq:appendix_mi_identity}
\end{equation}
\end{proposition}

\begin{proof}
Since \(Z=e(X)\) is a deterministic function of \(X\), we have:
\[
H(Z \mid X)=0
\qquad\text{and thus also }\qquad
H(Z \mid (X,C))=0.
\]
Hence:
\[
I(X;Z)=H(Z)-H(Z \mid X)=H(Z)=H(Z)-H(Z \mid (X,C))=I((X,C);Z).
\]
Applying the chain rule gives
\[
I((X,C);Z)=I(C;Z)+I(X;Z \mid C). 
\]
Because \(Z \perp\!\!\!\perp C\), we have \(I(C;Z)=0\), and therefore:
\begin{equation}
I(X;Z)=I(X;Z \mid C). \label{eq:appendix_mi_identity1}
\end{equation}
Using the standard identity for conditional mutual information,
\[
I(X;Z \mid C)=H(X \mid C)-H(X \mid Z,C),
\]
we obtain \eqref{eq:appendix_mi_identity}.
\end{proof}

\begin{corollary}
Under the constraint \(Z \perp\!\!\!\perp C\), maximizing \(I(X;Z)\) over the choice of $Z$ is equivalent to minimizing
\(H(X \mid Z,C)\). In particular, if there exists a decoder \(d\) such that:
\[
X = d(Z,C) \qquad \text{a.s.},
\]
then:
\[
H(X \mid Z,C)=0
\qquad\text{and}\qquad
I(X;Z)=H(X \mid C),
\]
which is the largest possible value of \(I(X;Z)\) among all \(C\)-independent deterministic
representations of \(X\).
\end{corollary}

\begin{proof}
The first claim follows immediately from Proposition~\ref{prop_2}, since \(H(X \mid C)\) is fixed by the
data-generating distribution and does not depend on the encoder. For the second claim, if
\(X=d(Z,C)\) almost surely (perfect reconstruction), then \(X\) is a deterministic function of \((Z,C)\), and therefore
\(H(X \mid Z,C)=0\). Substituting into \eqref{eq:appendix_mi_identity} yields:
\[
I(X;Z)=H(X \mid C).
\]
Finally, from ~\eqref{eq:appendix_mi_identity1} and the definition of conditional mutual information
\[
I(X;Z) = I(X;Z|C) = H(X|C) - H(X|Z,C) \le H(X|C),
\]
so this value is maximal.

\end{proof}

\begin{proposition}
\label{prop:counter_example}
In general, mutual information is not a monotonic function of $\hsic$. In particular, for fixed
Gaussian kernels there does not exist a universal monotone function $\phi$ such that
\[
I(X;Z)=\phi(\hsic(X,Z))
\]
for all joint distributions of $(X,Z)$.
\end{proposition}

\begin{proof}
We prove the stronger statement that, for fixed Gaussian kernels, there does not exist
\emph{any} universal function $\phi$ such that:
\[
I(X;Z)=\phi(\hsic(X,Z))
\]
for all joint distributions of $(X,Z)$.

Fix Gaussian kernels:
\[
k_\tau(x,x')=\exp\!\left(-\frac{\|x-x'\|_2^2}{\tau^2}\right),
\qquad
\ell_\sigma(z,z')=\exp\!\left(-\frac{\|z-z'\|_2^2}{\sigma^2}\right).
\]
Now choose two dependent pairs $(X_1,Z_1)$ and $(X_2,Z_2)$ such that:
\[
I(X_1;Z_1)\neq I(X_2;Z_2).
\]
For example, one may take two jointly Gaussian pairs with different correlations.

For $j\in\{1,2\}$ and $a>0$, define:
\[
T_a(z)=az,
\qquad
Z_{j,a}:=T_a(Z_j)=aZ_j.
\]
Since $T_a$ is bijective for every $a>0$,
\[
I(X_j;Z_{j,a})=I(X_j;Z_j)
\qquad\text{for all }a>0.
\]

We now study how $\hsic$ changes with $a$. Let $(X_j',Z_j')$ and $(X_j'',Z_j'')$ be i.i.d.\ copies
of $(X_j,Z_j)$, independent of each other and of $(X_j,Z_j)$. Writing $\hsic$ in population form,
\begin{align}
\hsic(X_j,Z_{j,a})
&=
\mathbb{E}\!\left[k_\tau(X_j,X_j')\,\ell_\sigma(aZ_j,aZ_j')\right] \nonumber\\
&+\mathbb{E}\!\left[k_\tau(X_j,X_j')\right]\,
 \mathbb{E}\!\left[\ell_\sigma(aZ_j,aZ_j')\right] \nonumber\\
&-2\,\mathbb{E}\!\left[k_\tau(X_j,X_j')\,\ell_\sigma(aZ_j,aZ_j'')\right].
\label{eq:hsic-population-scaled}
\end{align}

For every fixed $z,z'$,
\[
\ell_\sigma(az,az')
=
\exp\!\left(-\frac{a^2\|z-z'\|_2^2}{\sigma^2}\right)
\to 1
\qquad\text{as }a\to 0.
\]
Moreover,
\[
0\le \ell_\sigma(az,az')\le 1,
\qquad
0\le k_\tau(x,x')\le 1.
\]
Hence dominated convergence applies to \eqref{eq:hsic-population-scaled} and
\begin{align*}
\lim_{a\rightarrow 0}\hsic(X_j,Z_{j,a})
&=
\mathbb{E}[k_\tau(X_j,X_j')]
+\mathbb{E}[k_\tau(X_j,X_j')]
-2\mathbb{E}[k_\tau(X_j,X_j')] =0 
\end{align*}

For each fixed $j$, define:
\[
F_j(a):=\hsic(X_j,Z_{j,a}), \qquad a\in[0,1].
\]
Clearly, $F_j$ is continuous on $[0,1]$. 

At $a=1$ we recover the original pair $\hsic(X_j,Z_j)$. Since $\hsic$ is a maximum mean discrepancy,
\[
\hsic(X_j,Z_j)=0
\quad\Longleftrightarrow\quad
X_j\perp\!\!\!\perp Z_j.
\]
Because each $(X_j,Z_j)$ was chosen dependent, it follows that $\hsic(X_j,Z_j)>0.$

Since $\lim_{a\rightarrow 0}F_j(a)= 0$ and $F_j(1)=\hsic(X_j,Z_j)>0$, continuity implies that
$F_j$ takes every value in the interval $(0,\hsic(X_j,Z_j)]$.

Choose any
\[
h\in \bigl(0,\min\{\hsic(X_1,Z_1),\hsic(X_2,Z_2)\}\bigr).
\]
For each $j\in\{1,2\}$, we have:
\[
0<h<\min\{\hsic(X_1,Z_1),\hsic(X_2,Z_2)\}\le \hsic(X_j,Z_j)=F_j(1).
\]
Since  $F_j$ takes every value in the interval $(0,\hsic(X_j,Z_j)],$
there exists $a_j\in(0,1]$ such that $\hsic(X_j,Z_{j,a_j}) = F_j(a_j)=h$.
As mutual information is invariant under the invertible scaling $z\mapsto a_j z$,
\[
I(X_j;Z_{j,a_j})=I(X_j;Z_j), \qquad j\in\{1,2\}.
\]
Therefore,
\[
I(X_1;Z_{1,a_1})=I(X_1;Z_1)\neq I(X_2;Z_2)=I(X_2;Z_{2,a_2}).
\]
However:
\[
\hsic(X_1,Z_{1,a_1})=\hsic(X_2,Z_{2,a_2})=h.
\]
Hence we have constructed two joint distributions with the same $\hsic$ value but different mutual
information values. Therefore no universal function $\phi$ can satisfy:
\[
I(X;Z)=\phi(\hsic(X,Z))
\]
for all joint distributions. In particular, no universal monotone such function exists.

\end{proof}

\paragraph{Consequence for the preservation objective.}
Proposition~\ref{prop:counter_example} shows that maximizing \(\mathrm{HSIC}(X,Z)\) is, in general, a kernel-dependent
heuristic rather than the information-theoretically correct preservation objective. This is precisely
why we use reconstruction, equivalently, minimization of \(H(X \mid Z,C)\) under
\(Z \perp\!\!\!\perp C\), instead of an HSIC-based preservation term.

\subsection{Proofs for the Perfect Regime}
\label{app:perfect_regime_proofs}

\begin{proof}[Proof of Theorem~\ref{thm:perfect_preservation}]
Under \(Z \perp\!\!\!\perp C\), Proposition~\ref{prop_2} gives:
\[
I(X;Z) = H(X \mid C) - H(X \mid Z, C).
\]
Since $\hat X = d(Z,C) = X$ almost surely, \(X\) is a deterministic function of \((Z,C)\), and therefore $H(X \mid Z, C) = 0$.
Hence $I(X;Z) = H(X \mid C)$.

It remains to identify \(H(X \mid C)\). Since \(f\) is injective on
\(\mathcal U \times \mathcal C\), for every fixed \(c \in \mathcal C\)
the map \(f_c(u) := f(u,c)\) is a bijection from \(\operatorname{supp}(U)\)
onto \(\operatorname{supp}(X \mid C = c)\), so $H(X \mid C = c) = H(U \mid C = c)$.
Since \(U \perp\!\!\!\perp C\), we obtain $H(U \mid C) = H(U)$,
and therefore $H(X \mid C) = H(U)$.

\end{proof}

\begin{proof}[Proof of Corollary~\ref{cor:optimality_c_independent}]
Let \(Z' = e'(X)\) be any deterministic representation satisfying \(I(C;Z')=0\).
Since \(Z'\) is a deterministic function of \(X\) and is independent of \(C\), Proposition~\ref{prop_2} applies to \(Z'\), yields
\[
I(X;Z') = H(X\mid C) - H(X\mid Z',C).
\]
As conditional entropy is nonnegative, we have $I(X;Z') \le H(X\mid C)$.
On the other hand, by Theorem~\ref{thm:perfect_preservation}, $I(X;Z)=H(X\mid C)$.
Combining the two displays gives $I(X;Z') \le I(X;Z)$
as claimed.

\end{proof}

\begin{proof}[Proof of Corollary~\ref{cor:perfect_guardedness}]
The conditions in Definition~\ref{def:perfect_guardedness}
follow directly from Corollary~\ref{cor:optimality_c_independent}.

\end{proof}

\begin{corollary}[Deterministic recovery of the complementary factor]
\label{cor:recover_u}
Under the assumptions of Theorem~\ref{thm:perfect_preservation},
\[
H(U\mid Z,C)=0.
\]
\end{corollary}

\begin{proof}
Since \(f\) is injective and \(X=f(U,C)\), there exists a measurable map $h:\operatorname{Im}(f)\to \mathcal U\times \mathcal C$
such that $h(X)=(U,C)$ almost surely.
Let \(\pi_U:\mathcal U\times\mathcal C\to\mathcal U\) be the projection onto the first coordinate, and define:
\[
\eta(z,c):=\pi_U(h(d(z,c))).
\]
Because \(d(Z,C)=X\) we have $\eta(Z,C)=\pi_U(h(d(Z,C)))=\pi_U(h(X))=U$. 
Thus \(U\) is a deterministic function of \((Z,C)\), and therefore $H(U\mid Z,C)=0$.

\end{proof}

Thus, perfect reconstruction implies that no complementary information is lost: once \(C\) is provided, \(U\) can be recovered exactly from \(Z\).

\begin{corollary}[Exact preservation of complementary information]
\label{}
Under the assumptions of Theorem~\ref{thm:perfect_preservation},
\[
I(U; Z \mid C) = H(U).
\]
\end{corollary}

\begin{proof}
By Corollary~\ref{cor:recover_u}, $H(U \mid Z, C) = 0.$ 
Using \(U \perp\!\!\!\perp C\), we have \(H(U \mid C) = H(U)\), and therefore:
\[
I(U;Z \mid C) = H(U \mid C) - H(U \mid Z, C) = H(U). \qedhere
\]
\end{proof}

Combined with $H(U \mid Z, C) = 0$ from
Corollary~\ref{cor:recover_u}, the following corollary shows that $Z$ is related
to $U$ by an entropy-preserving transformation, recovering the
identifiability-up-to-entropy-preservation property.
\begin{corollary}[Entropy-preserving representation]
\label{cor:entropy_preserving}
Under the assumptions of Theorem~\ref{thm:perfect_preservation},
\[
H(Z) = H(U).
\]

\end{corollary}
\begin{proof}
Since \(Z = e(X)\) is a deterministic function of \(X\), we have $H(Z \mid X) = 0$, and therefore $H(Z) = I(X;Z)$. 
By Theorem~\ref{thm:perfect_preservation}, \(I(X;Z) = H(U)\), so \(H(Z) = H(U)\).

\end{proof}

\subsection{Proofs for the Imperfect Regime}
\label{app:imperfect_regime}
\begin{proof}[Proof of Proposition~\ref{prop:mse_implies_preservation}]
By the defining property of \(h\) in Assumption~\ref{ass:lipschitz_left_inverse}, we have $U = h(f(U,C),C) = h(X,C)$.
By definition, $\hat U = h(d(Z,C),C)$.
Therefore $\|U-\hat U\|_2=\|h(X,C)-h(d(Z,C),C)\|_2$.
Applying the Lipschitz property of \(h\) with the same concept value \(C\) gives $\|U-\hat U\|_2 \le L_h\,\|X-d(Z,C)\|_2$. 
Squaring both sides yields $\|U-\hat U\|_2^2\le L_h^2\,\|X-d(Z,C)\|_2^2$.
Taking expectation, we obtain $\mathbb E\!\left[\|U-\hat U\|_2^2\right]\le L_h^2\,\mathbb E\!\left[\|X-d(Z,C)\|_2^2\right] \le L_h^2\,\varepsilon_{\mathrm{rec}}$,
as claimed.

\end{proof}

\begin{corollary}[Finite-alphabet preservation bound from MSE]
\label{cor:finite_alphabet_guardedness}
Assume the conditions of Proposition~\ref{prop:mse_implies_preservation}. In addition, assume that
the alphabet \(\mathcal U\) of \(U\) is finite with \(|\mathcal U|\ge 2\). Let $\Delta_U := \min_{u\neq u'} \|u-u'\|_2 > 0$. 
In addition, let $U^{\mathrm{dec}} := \arg\min_{u'\in \operatorname{supp}(U)}\|\hat U-u'\|_2$ be the nearest neighbor element of \(U\) from \(\hat U\), and define $\delta_{\mathrm{mse}}:=\min\left\{1-\frac{1}{|\mathcal U|},\,\frac{4L_h^2\varepsilon_{\mathrm{rec}}}{\Delta_U^2}\right\}$.
Then:
\[
\mathbb P\!\left[U^{\mathrm{dec}}\neq U\right]
\le
\frac{4L_h^2\varepsilon_{\mathrm{rec}}}{\Delta_U^2}.
\]
Consequently,
$I(U;Z\mid C)\ge H(U) - h_2(\delta_{\mathrm{mse}}) - \delta_{\mathrm{mse}}\log(|\mathcal U|-1)$,
where \(h_2(p):=-p\log p-(1-p)\log(1-p)\) is the binary entropy function.
\end{corollary}

\begin{proof}
If \(U^{\mathrm{dec}}\neq U\), then some \(u'\neq U\) is at least as close to \(\hat U\) as \(U\) is.
Since any two distinct points in \(\operatorname{supp}(U)\) are separated by at least \(\Delta_U\), this is only possible if $\|\hat U-U\|_2 \ge \frac{\Delta_U}{2}$.
Hence:
\[
\mathbb P\!\left[U^{\mathrm{dec}}\neq U\right]
\le
\mathbb P\!\left[\|\hat U-U\|_2^2 \ge \frac{\Delta_U^2}{4}\right].
\]
Applying Markov's inequality to the right hand side gives:
\[
\mathbb P\!\left[U^{\mathrm{dec}}\neq U\right]
\le
\frac{4\,\mathbb E[\|\hat U-U\|_2^2]}{\Delta_U^2}.
\]
By Proposition~\ref{prop:mse_implies_preservation}, $\mathbb E[\|\hat U-U\|_2^2] \le L_h^2\varepsilon_{\mathrm{rec}}$,
and therefore $\mathbb P\!\left[U^{\mathrm{dec}}\neq U\right] \le \frac{4L_h^2\varepsilon_{\mathrm{rec}}}{\Delta_U^2}$.

Write $P_e:=\mathbb P\!\left[U^{\mathrm{dec}}\neq U\right]$ and
$f(p):=h_2(p)+p\log(|\mathcal U|-1)$. Since \(U\) is finite, Fano's inequality
gives \(H(U\mid Z,C)\le f(P_e)\). The function \(f\) increases on
\([0,\,1-1/|\mathcal U|]\) and reaches its maximum \(\log|\mathcal U|\) at
\(p=1-1/|\mathcal U|\). If \(P_e\le 1-1/|\mathcal U|\), then
\(P_e\le\delta_{\mathrm{mse}}\) and \(f(P_e)\le f(\delta_{\mathrm{mse}})\).
Otherwise the bound above forces
\(\delta_{\mathrm{mse}}=1-1/|\mathcal U|\), so
\(f(\delta_{\mathrm{mse}})=\log|\mathcal U|\ge H(U\mid Z,C)\). In both cases
\[
H(U\mid Z,C)
\le
h_2(\delta_{\mathrm{mse}})
+
\delta_{\mathrm{mse}}\log(|\mathcal U|-1).
\]
Finally, using \(U\perp\!\!\!\perp C\), we have \(H(U\mid C)=H(U)\), and therefore:
\[
I(U;Z\mid C)
=
H(U\mid C)-H(U\mid Z,C)
\ge
H(U)
-
h_2(\delta_{\mathrm{mse}})
-
\delta_{\mathrm{mse}}\log(|\mathcal U|-1),
\]
which proves the claim.
\end{proof}

\subsection{Proofs for Concept Conversion}
\label{app:conversion-proof}

\begin{proof}[Proof of Corollary~\ref{cor:conversion}]
By assumption,
\[
d(e(X), C) = X \ \text{a.s.}, \qquad Z = e(X) \perp\!\!\!\perp C.
\]
For each $c \in \mathrm{supp}(C)$, define:
\[
T_c(u) := e(f(u, c)).
\]
We first show that each $T_c$ is injective. Fix $c$ and suppose $T_c(u_1) = T_c(u_2)$. Then $d(T_c(u_1), c) = d(T_c(u_2), c)$, and by the definition of $T_c$ together with perfect reconstruction,
\[
d(T_c(u_i), c) = d(e(f(u_i, c)), c) = f(u_i, c), \qquad i = 1, 2.
\]
Hence $f(u_1, c) = f(u_2, c)$, and since $f$ is injective on $\mathcal{U} \times \mathcal{C}$, it follows that $u_1 = u_2$.

Since $U \perp\!\!\!\perp C$ in the structural model $X = f(U,C)$, the conditional
law of $U$ given $C = c$ coincides with its marginal: $P_{U \mid C=c} = P_U$.
Therefore, conditional on $\{C = c\}$,
\[
  Z \;=\; e(X) \;=\; e\bigl(f(U, c)\bigr) \;=\; T_c(U),
\]
and the conditional law of $Z$ given $\{C = c\}$ is the pushforward of $P_U$
under $T_c$:
\[
  P_{Z \mid C = c} \;=\; (T_c)_{\#}\, P_U.
\]
By the independence assumption $Z \perp\!\!\!\perp C$, the conditional law
$P_{Z \mid C = c}$ does not depend on $c$ and equals the marginal $P_Z$.
Hence, for any $c, c' \in \mathrm{supp}(C)$,
\[
  (T_c)_{\#}\, P_U
  \;=\; P_{Z \mid C = c}
  \;=\; P_Z
  \;=\; P_{Z \mid C = c'}
  \;=\; (T_{c'})_{\#}\, P_U.
\]
We now extract a single map from the family $\{T_c\}_{c \in \text{supp}(C)}$.
Fix arbitrary $c, c' \in \mathrm{supp}(C)$; we show $T_c = T_{c'}$.

In the discrete setting, the support of a pushforward measure is exactly the
image of the underlying map. Combined with the previous display, this gives
\[
  T_c\bigl(\mathrm{supp}(U)\bigr)
  \;=\; \mathrm{supp}\bigl((T_c)_{\#} P_U\bigr)
  \;=\; \mathrm{supp}\bigl((T_{c'})_{\#} P_U\bigr)
  \;=\; T_{c'}\bigl(\mathrm{supp}(U)\bigr)
  \;=:\; S,
\]
so $T_c$ and $T_{c'}$ are both bijections from $\mathrm{supp}(U)$ onto the
common image $S$, and $T_{c'}^{-1} : S \to \mathrm{supp}(U)$ is well-defined.
Define:
\[
  \pi_{c,c'} \;:=\; T_{c'}^{-1} \circ T_c,
\]
which, as a composition of bijections, is a bijection from $\mathrm{supp}(U)$
to itself. Pushing the equality $(T_c)_{\#} P_U = (T_{c'})_{\#} P_U$ forward
by $T_{c'}^{-1}$ yields:
\[
  (\pi_{c,c'})_{\#}\, P_U
  \;=\; (T_{c'}^{-1})_{\#}\, (T_c)_{\#}\, P_U
  \;=\; (T_{c'}^{-1})_{\#}\, (T_{c'})_{\#}\, P_U
  \;=\; P_U,
\]
where the last equality follows because \(T_{c'}^{-1}\circ T_{c'}\) is the identity on
\(\operatorname{supp}(U)\).
Thus \(\pi_{c,c'}\) is a measure-preserving automorphism of \(\operatorname{supp}(U)\).
By Assumption~\ref{ass:asym_u}, the only such automorphism is the identity,
hence $\pi_{c,c'} = \mathrm{id}_{\mathrm{supp}(U)}$, i.e.\ $T_c = T_{c'}$.
Since $c, c'$ were arbitrary, $T_c$ does not depend on $c$, and we may set
$T := T_c$ for any $c \in \mathrm{supp}(C)$. With this definition,
$T_c(u) = T(u)$ for every $u \in \mathrm{supp}(U)$ and every
$c \in \mathrm{supp}(C)$.

Combining $T(\mathrm{supp}(U)) = \mathrm{supp}(T_{\#} P_U)$ with
$T_{\#} P_U = P_Z$ shows that $T$ is a bijection from $\mathrm{supp}(U)$
onto $\mathrm{supp}(Z)$.

Finally, fix any target concept value $c' \in \mathrm{supp}(C)$. Since $e(f(U, C)) = T(U)$ a.s., we have:
\[
d(e(f(U, C)), c') = d(T(U), c').
\]
Applying perfect reconstruction to the sample $f(U, c')$ gives:
\[
d(e(f(U, c')), c') = f(U, c') \ \text{a.s.},
\]
and since $e(f(U, c')) = T(U)$, this becomes:
\[
d(T(U), c') = f(U, c') \ \text{a.s.}
\]
Combining the last two equations yields:
\[
d(e(f(U, C)), c') = f(U, c') \ \text{a.s.},
\]
which is the desired concept-conversion statement.
\end{proof}

\section{Empirical Robustness to Violated Independence}\label{app:robustness}

To probe how much our framework relies on \(U \perp\!\!\!\perp C\), we sweep a
controlled synthetic process \(X = f(U, C)\) in which the dependence between
\(U\) and \(C\) is governed by a single scalar \(\alpha \in [0, 1]\):
\(\alpha = 0\) recovers the paper's assumption, \(\alpha = 1\) makes \(C\) a
deterministic function of \(U\), and intermediate values interpolate. For
each \(\alpha\) we train the same model with the same loss (Eq.~2) and
evaluate erasure, utility, and steering.

\paragraph{Takeaway.}
Even when perfect independence does not hold, at reasonable levels of
dependence the framework continues to work well empirically. Erasure is
achieved at the chance floor for every \(\alpha\); the cost shows up as
a gradual loss of utility on the entangled coordinate of \(U\) only,
and at small and moderate dependence levels (e.g.\ \(\alpha \le 0.5\))
most of that utility is still retained. Steering is fully successful
from a downstream classifier's perspective at every \(\alpha\). In
short: when the deviation from independence is small, the downgrade in
performance is small.

\subsection{Setup}\label{app:robustness:setup}

\paragraph{Data.} The latent is \(U \in \mathbb{R}^{3}\) with
\(U \sim \mathcal{N}(0, I)\); we write \(U_{0}\) for its first
coordinate, \(U_{1}\) for its second, and \(U_{2}\) for its third.
Intuitively, the sensitive concept \(C\) is constructed to depend on
the first coordinate \(U_{0}\), with the magnitude of \(\alpha\)
controlling the strength of that dependence. Concretely, we add noise to \(U_{0}\) by an independent
\(\varepsilon \sim \mathcal{N}(0, 1)\) to form a continuous score:
\[
s \;=\; \alpha\, U_{0} \;+\; \sqrt{1 - \alpha^{2}}\,\varepsilon ,
\]
and assign \(C\) by binning \(s\) into the four equiprobable quantile
bins of its empirical distribution: we compute the 25th, 50th and 75th
percentiles \(q_{1}, q_{2}, q_{3}\) of \(s\), and set \(C = 0\) if
\(s < q_{1}\), \(C = 1\) if \(q_{1} \le s < q_{2}\), \(C = 2\) if
\(q_{2} \le s < q_{3}\), and \(C = 3\) if \(s \ge q_{3}\). The factor
\(\sqrt{1 - \alpha^{2}}\) keeps \(\mathrm{Var}(s) = 1\) for every
\(\alpha\), so the bin edges and class balance are identical across the
sweep; only the joint \(P(U, C)\) changes. By construction only \(U_{0}\)
is entangled with \(C\) (with strength controlled by \(\alpha\)), while
\(U_{1}, U_{2}\) are \(C\)-independent for every \(\alpha\). This split
lets us cleanly distinguish utility the model is forced to discard from
utility it should preserve.

\paragraph{Generative map \(f\).} A frozen, randomly initialized injective
map \(f : \mathbb{R}^{3} \times \{0,\dots,3\} \to \mathbb{R}^{16}\),
\[
f(U, C) \;=\; \big[\,U,\;\; \phi(C),\;\; \mathrm{MLP}([U, \phi(C)])\,\big],
\]
where \(\phi : \{0,\dots,K-1\} \to \mathbb{R}^{4}\) is a fixed class
embedding -- a lookup table whose \(K\) rows are sampled once from
\(\mathcal{N}(0, 1.5^{2})\) at construction and then held fixed (rows
distinct with probability one), so \(\phi(c)\) is just the \(c\)-th row of
this fixed matrix -- and \(\mathrm{MLP}\) is a two-hidden-layer Tanh
network. The skip-concatenation makes \(f\) provably
injective in \((U, C)\) and gives a 1-Lipschitz left-inverse
\(h(X, C) = X_{0:3}\), so Assumption~\ref{ass:lipschitz_left_inverse} holds with \(L_{h} = 1\). For each
\(\alpha\) we sample \(N_{\text{train}} = 6{,}000\) and
\(N_{\text{test}} = 2{,}000\) i.i.d.\ examples.

\paragraph{Model and training.} Encoder
\(e : \mathbb{R}^{16} \to \mathbb{R}^{4}\) (two hidden layers of width 64,
GELU); decoder \(d : \mathbb{R}^{4} \times \{0,\dots,3\} \to \mathbb{R}^{16}\)
(concept embedded into \(\mathbb{R}^{4}\), concatenated with \(Z\), two
hidden layers of width 64). Trained for 300 epochs with AdamW (learning rate
\(2 \times 10^{-3}\), batch 1024) using exactly Eq.~(2), with
\(\lambda_{\text{ind}} = 50\). HSIC is estimated with an RBF kernel on \(Z\)
(median-heuristic bandwidth, recomputed per minibatch) and a linear kernel
on the one-hot \(C\). We sweep
\(\alpha \in \{0, 0.10, 0.25, 0.35, 0.50, 0.60, 0.75, 0.85, 1\}\); a fresh
model and fresh probes are trained for each \(\alpha\).

\subsection{Metrics}\label{app:robustness:metrics}

\paragraph{Erasure.} An MLP probe (two hidden layers of width 64, GELU,
200 AdamW epochs) trained to predict \(C\). To establish the headroom we
report its accuracy on the frozen \(X\); to measure erasure we report the
accuracy of the same probe trained on \(Z\). Lower is better; chance is
\(1/K = 0.25\).

\paragraph{Utility.} An MLP regressor (same architecture, MSE loss) trained
to predict \(U\) from \(Z\). We report
\(R^{2} = 1 - \tfrac{\sum_i \|U_i - \hat U_i\|_{2}^{2}}{\sum_i \|U_i - \bar U\|_{2}^{2}}\),
the fraction of target variance explained (\(1\) is perfect, \(0\) is no
better than predicting the mean), separately for \(U_{0}\) (the
\(C\)-entangled coordinate) and \(U_{1:} = (U_{1}, U_{2})\) (the
\(C\)-independent coordinates). The split isolates utility that
\emph{must} be discarded under the constraint from utility that
\emph{should} be preserved.

\paragraph{Steering.} Given test \(x = f(U, C)\) and a target \(c'\), the
converted representation is \(x_{\text{conv}} = d(e(x), c')\). We measure
two things. \emph{(i) Oracle counterfactual MSE}:
\(\|x_{\text{conv}} - f(U, c')\|_{2}^{2}\), with the do-nothing baseline
\(\|x - f(U, c')\|_{2}^{2}\) for scale -- this is the strict point-wise
test, available only because we synthesized \(f\). \emph{(ii) Probe-based
success rate}, mirroring the steering evaluation in our main experiments:
a frozen-\(X\) classifier is applied to \(x_{\text{conv}}\) and we report
the rate at which it predicts \(c'\), with the rate of predicting \(c'\) on
the unmodified \(x\) as the trivial baseline.

\begin{figure}[t]
  \centering
  \includegraphics[width=\linewidth]{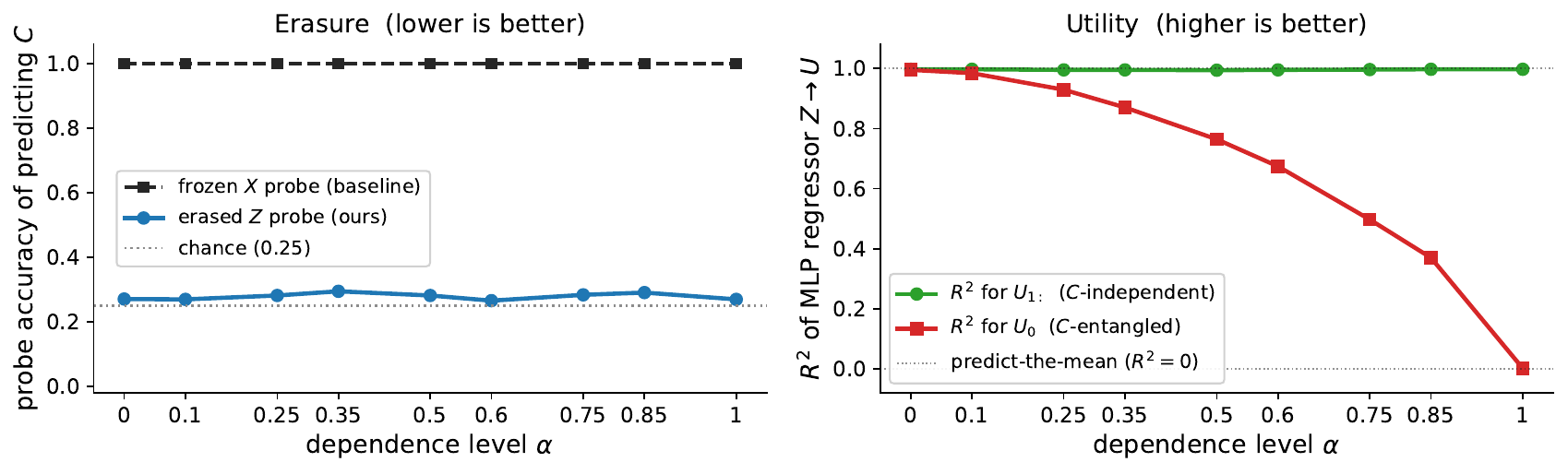}
  \caption{Erasure and utility under varying \(U\)--\(C\) dependence.
Left: concept leakage measured by an MLP probe predicting \(C\). The frozen
representation \(X\) exposes \(C\), while the erased representation \(Z\) stays
near chance across all dependence levels. Right: utility measured by
\(R^{2}\) for predicting \(U\) from \(Z\). The \(C\)-independent coordinates
\(U_{1:}=(U_{1},U_{2})\) remain preserved, while the \(C\)-entangled coordinate
\(U_{0}\) degrades gradually as \(\alpha\) increases.
}
  \label{fig:erasure_utility_combined_combined}
\end{figure}

\subsection{Erasure and Utility}\label{app:robustness:tradeoff}
Figure~\ref{fig:erasure_utility_combined_combined} summarizes the effect of
violating the independence assumption \(U \perp\!\!\!\perp C\). The key
observation is that the method remains robust across the sweep: the erased
representation \(Z\) stays essentially non-informative about \(C\), while utility
degrades in a controlled and interpretable way.

The left panel shows that concept leakage remains at chance level even as
\(\alpha\) increases and \(C\) becomes more strongly dependent on \(U_{0}\).
Thus, the independence penalty continues to enforce erasure even outside the
ideal theoretical regime. The right panel shows where the cost of this erasure
appears. The \(C\)-independent coordinates \(U_{1:}=(U_{1},U_{2})\) remain almost
perfectly recoverable from \(Z\) throughout the sweep, whereas the
\(C\)-entangled coordinate \(U_{0}\) degrades as its dependence on \(C\)
increases.

This is the desired robustness behavior. When \(\alpha\) is small or moderate,
the model still preserves substantial information about \(U_{0}\), so mild
violations of \(U \perp\!\!\!\perp C\) do not cause a collapse in utility.
Moreover, the transition is smooth: performance does not suddenly break once
the independence assumption is violated, but instead degrades progressively as
the statistical overlap between \(U_{0}\) and \(C\) increases. As \(\alpha\)
becomes large, information about \(U_{0}\) increasingly overlaps with
information about \(C\), so preserving all of \(U_{0}\) would necessarily
reintroduce concept leakage. The observed degradation is therefore gradual and
localized: the model gives up the information that conflicts with independence
at a rate that reflects the strength of the violation, while continuing to
preserve the complementary information that does not.

\subsection{Steering}\label{app:robustness:conversion}

Table~\ref{tab:robust:conversion} evaluates whether the learned decoder still
performs steering when the independence assumption is gradually
violated. The probe-based metric remains perfect across the entire sweep:
after steering, a classifier trained on frozen representations identifies the
converted sample as belonging to the target concept value \(c'\) for every
\(\alpha\). This shows that the operational steering behavior is robust even
when \(U\) and \(C\) are not perfectly independent.

The stricter oracle metric reveals the expected cost of the violation. As
\(\alpha\) increases, \(C\) becomes increasingly tied to \(U_{0}\), so converting
to a different concept value \(c'\) while preserving the same \(U\) becomes
harder. Nevertheless, the degradation is gradual rather than abrupt: the oracle
MSE increases smoothly from \(0.011\) at \(\alpha=0\) to \(0.206\) at
\(\alpha=1\), while remaining far below the do-nothing baseline throughout.
Thus, even for steering, moderate violations of \(U \perp\!\!\!\perp C\)
lead to moderate degradation rather than failure.

\begin{table}[h]
  \centering
  \caption{We report oracle counterfactual MSE
\(\|d(e(x),c') - f(U,c')\|_{2}^{2}\), the corresponding do-nothing baseline
\(\|x - f(U,c')\|_{2}^{2}\), and probe-based steering success measured by a
frozen-\(X\) classifier. Lower oracle MSE and higher probe success indicate
better steering. Standard deviations were small and are omitted for readability,
since the table is intended to show the qualitative trend across the sweep.}
  \label{tab:robust:conversion}
  \small
  \begin{tabular}{cccccc}
    \toprule
    & \multicolumn{2}{c}{Oracle MSE} & \multicolumn{3}{c}{Probe-based success} \\
    \cmidrule(lr){2-3}\cmidrule(lr){4-6}
    \(\alpha\) &
    conv MSE \(\downarrow\) & do-nothing &
    \(p \to c\) on \(x\) (sanity) &
    \(p \to c'\) on \(x_{\text{conv}}\) \(\uparrow\) &
    \(p \to c'\) on \(x\) (baseline) \\
    \midrule
    0.00 & 0.011 & 1.530 & 1.000 & 1.000 & 0.000 \\
    0.10 & 0.014 & 1.530 & 1.000 & 1.000 & 0.000 \\
    0.25 & 0.020 & 1.531 & 1.000 & 1.000 & 0.000 \\
    0.35 & 0.037 & 1.532 & 1.000 & 1.000 & 0.000 \\
    0.50 & 0.056 & 1.533 & 1.000 & 1.000 & 0.000 \\
    0.60 & 0.074 & 1.534 & 1.000 & 1.000 & 0.000 \\
    0.75 & 0.118 & 1.537 & 1.000 & 1.000 & 0.000 \\
    0.85 & 0.143 & 1.540 & 1.000 & 1.000 & 0.000 \\
    1.00 & 0.206 & 1.544 & 1.000 & 1.000 & 0.000 \\
    \bottomrule
  \end{tabular}
\end{table}

\section{Implementation Details}
\label{app:implementation}

\subsection{Datasets}
\label{app:datasets}
We evaluate on three standard concept-erasure benchmarks: \textsc{Bias in Bios},
\textsc{Dial-Sentiment}, and \textsc{Dial-Mention}. In all experiments, the input
to our method is a frozen pretrained-language-model representation, the concept
label is used during erasure, and the downstream label is used only for evaluation.

\paragraph{\textsc{Bias in Bios}.}
\textsc{Bias in Bios}~\cite{deArteaga2019bias} consists of short biographies,
each annotated with a profession label and a binary gender label. We use
profession as the downstream target $Y$ and gender as the sensitive concept $C$.
The dataset contains 28 profession classes and two gender groups. In the processed
split used in our experiments, $53.7\%$ of the examples are labeled male and
$46.3\%$ female. The largest profession class is \textit{professor}, which accounts
for approximately $30\%$ of the samples. This benchmark is therefore a multiclass
utility-preservation setting, where the goal is to remove gender information while
retaining profession-relevant information.

\paragraph{\textsc{Dial-Sentiment}.}
\textsc{Dial-Sentiment} is derived from the dialectal Twitter corpus of
\cite{blodgett2016demographic}. The downstream target $Y$ is binary sentiment,
with $54.57\%$ happy examples and $45.43\%$ sad examples. The sensitive concept
$C$ is race, with $36.93\%$ African-American English (AAE) examples and
$63.07\%$ Standard American English (SAE) examples.

\paragraph{\textsc{Dial-Mention}.}
\textsc{Dial-Mention} uses the same dialectal Twitter benchmark, but the downstream
target $Y$ is whether a tweet mentions another user. The target labels are balanced,
with $50\%$ conversational and $50\%$ non-conversational examples. The sensitive
concept $C$ is again race/dialect, and the AAE/SAE labels are also balanced
($50\%$--$50\%$).

\paragraph{Use of labels.}
Our method is trained without downstream-task labels. The erasure objective uses
only the representation $X$ and the concept label $C$. The downstream labels $Y$
are used only after training, for post-hoc evaluation of utility, fairness, and
steering performance.

\subsection{Experimental Setup}
\label{sec:exp_setup}
\paragraph{Architecture.}
The encoder $e\colon\mathcal{X}\to\mathcal{Z}$ maps the input representation
$x\in\mathbb{R}^{D}$ to a concept-free latent code $z\in\mathbb{R}^{512}$.
The decoder $d\colon\mathcal{Z}\times\mathcal{C}\to\mathcal{X}$ acts as its
conditional left-inverse: given the code $z$ and a concept label
$c\in\mathcal{C}$, it reconstructs a representation in the original input
space $\mathbb{R}^{D}$.
Both are two-layer MLPs with GELU activations; the decoder fuses $z$ with a
learnable concept embedding of dimension 64 via concatenation before projecting
back to $\mathbb{R}^{D}$.

\paragraph{Optimization.}
Encoder and decoder are trained jointly with a single NAdamW optimiser (learning rate $5\times10^{-4}$, weight decay $0.05$, batch size $16{,}384$).
Unlike adversarial methods, there is no minimax game and no alternating update schedule.
The training objective is
\begin{equation}
    \mathcal{L}
    \;=\;
    \lambda_{\mathrm{mse}}\,\mathrm{MSE}\!\bigl(x,\,d(e(x),c)\bigr)
    \;+\;
    \lambda_{\mathrm{ind}}\,\widehat{\mathrm{HSIC}}\!\bigl(e(x),\,c\bigr),
    \label{eq:loss}
\end{equation}
where $\lambda_{\mathrm{mse}}=1$ throughout.

\paragraph{HSIC estimation.}
Similar to Obliviator, $\widehat{\mathrm{HSIC}}$ is computed as the Frobenius norm of the empirical cross-covariance
between Random Fourier Feature (RFF) maps of the code and one-hot concept labels~\citep{rahimi2007random}.
Before feature mapping the code is $\ell_2$-normalised, preventing the encoder from
trivially deflating the HSIC estimate by scaling representation norms.
The RBF bandwidth $\sigma$ is recomputed each mini-batch via the median heuristic and
floored at $\sigma_{\min}$ to prevent kernel degeneracy.
The number of Fourier features is set to
$d_{\mathrm{rff}} = \mathrm{clamp}(32 d_z,\,1000,\,32768)$,
where $d_z=512$ is the code dimension.

\paragraph{Why random features.} We use RFF for memory and speed. At our
batch size of $16{,}384$, a single exact Gram matrix holds $2.7\times10^{8}$
entries, about $1$\,GiB in FP32 before centering and autograd storage, and
its cost grows quadratically with the batch; RFF grows linearly. We checked
that the approximation costs nothing measurable. Training with RFF and with the
exact Gram matrix at batch sizes $4{,}096$, $8{,}192$ and $16{,}384$ on all three
datasets, with everything else fixed, gives utility minus leakage within
$1.9$ points at every matched setting, with no consistent winner, and
indistinguishable held-out exact HSIC. Exact HSIC at batch $8{,}192$ against
RFF at $16{,}384$ differs by $3.2$ points on \textsc{Bias in Bios} and $2.8$ on
\textsc{Dial-Mention}, yet takes $4\times$ longer per epoch; at batch $16{,}384$
exact HSIC is $8\times$ slower and needs about $9.3$\,GiB.

We follow the Obliviator evaluation protocol for every method, ours included.
Each configuration is trained with 5 random seeds. Every MLP probe is trained
3 times from different initialisations, and a run's leakage is the maximum
accuracy over the restarts, so each run is scored against its strongest
adversary. Utility is measured with a single probe. Of the 5 seeds we report
the one with the best trade-off,
\begin{equation}
    \text{score} = \sum_{k} \bigl(\mathrm{Acc}^{\mathrm{util}}_k - \mathrm{Acc}^{\mathrm{leak}}_k\bigr),
    \label{eq:selection}
\end{equation}
where $k$ runs over the evaluation points of the curve. Leakage and utility
always come from the same seed. 

\paragraph{Hyperparameter selection.}
\begin{table}[h]
\centering
\caption{Per-dataset hyperparameters for our method. All other hyperparameters
are fixed globally.}
\label{tab:hparams}
\setlength{\tabcolsep}{8pt}
\begin{tabular}{lcccc}
\toprule
\textbf{Dataset} &
\textbf{Epochs} &
$\boldsymbol{\lambda_{\mathrm{mse}}}$ &
$\boldsymbol{\lambda_{\mathrm{ind}}}$ &
$\boldsymbol{\sigma_{\min}}$ \\
\midrule
\textsc{Bias in Bios}    & 500 & 1.0 & 5.0 & 0.05 \\
\textsc{DIAL-Sentiment}  & 500 & 1.0 & 5.0 & 0.05 \\
\textsc{Dial-Mention}    & 500 & 1.0 & 5.0 & 0.05 \\
\bottomrule
\end{tabular}
\end{table}
We tune the reconstruction weight $\lambda_{\mathrm{mse}}$ and the independence
weight $\lambda_{\mathrm{ind}}$ on the validation split; everything else is fixed
to the values above. The bandwidth floor is $\sigma_{\min}=0.05$ for every
dataset and every PLM. The selected values are in Table~\ref{tab:hparams}.

The same pair, $\lambda_{\mathrm{mse}}=1$ and $\lambda_{\mathrm{ind}}=5$, is used
on all three datasets. The grids behind each plotted curve are listed in
Table~\ref{tab:figure_hyperparams}.
Hyperparameter sensitivity results are provided in
Appendix~\ref{sec:hyperparameter_sens}.

\begin{table}[h]
\centering
\caption{
Hyperparameters used for our method in the plotted erasure results.
}
\label{tab:figure_hyperparams}
\small
\begin{tabular}{lcccc}
\toprule
Figure / setting & Epochs & \(\lambda_{\mathrm{mse}}\) & \(\lambda_{\mathrm{ind}}\) & \(\sigma_{\min}\) \\
\midrule
Fig.~\ref{fig:bios_bert} / BERT, \textsc{Bias in Bios}
& 500 & 1.0 & \(\{0.1,0.5,1,5\}\) & 0.05 \\

Fig.~\ref{fig:bios_plms}a / GPT-2, \textsc{Bias in Bios}
& 500 & 1.0 & \(\{1,5,10\}\) & 0.05 \\

Fig.~\ref{fig:bios_plms}b / DeepSeek, \textsc{Bias in Bios}
& 500 & 1.0 & \(\{1,2,5,10\}\) & 0.05 \\

Fig.~\ref{fig:bios_plms}c / LLaMA, \textsc{Bias in Bios}
& 500 & 1.0 & \(\{1,5,10\}\) & 0.05 \\

Fig.~\ref{fig:dial_bert}a / \textsc{DIAL-Mention}
& 500 & 1.0 & \(\{5,10\}\) & 0.05 \\

Fig.~\ref{fig:dial_bert}b / \textsc{Dial-Sentiment}
& 500 & 1.0 & \(\{5,10\}\) & 0.05 \\

\bottomrule
\end{tabular}
\end{table}
\paragraph{Probing Classifiers.} All methods are evaluated using the same probing setup to ensure a fair comparison; these classifiers are used for both the
erasure and steering evaluations.
We use two nonlinear MLP probes: one trained to predict the sensitive attribute
$c$ (measuring residual concept leakage) and one trained to predict the utility
label $y$ (measuring retained task information).
Each probe is a five-hidden-layer MLP with 128 units per layer and GELU
activations, optimised with NAdamW (learning rate $3\times10^{-3}$, weight
decay $0.05$, batch size $4{,}096$) for up to 200 epochs.
This is strictly stronger than the single-hidden-layer MLP of 100 units with
ReLU provided by scikit-learn, which much prior concept-erasure
work has commonly used as its nonlinear adversary; a weaker probe would
underestimate residual leakage and inflate apparent erasure quality.

\subsection{Steering Setup}
\label{app:steering_setup}
In all three experiments the language model, its LM head and all probes stay
frozen; only the conditional autoencoder is trained, with the concept label
and no downstream label.

\paragraph{Toxicity (Qwen2.5-7B).}
A \textsc{RealToxicityPrompts} prompt is labelled toxic if its released score
is at least $0.5$ and non-toxic if it is at most $0.25$; prompts in between are
dropped. We remove repeated prefixes, balance the classes and stratify the
split, so chance is $50\%$. The autoencoder uses the final-layer,
post-normalisation state at the last prompt token; we checked that it
reproduces the model's logits. We use this point because an intervention on
an earlier block is partly undone by later attention over unmodified
positions. The hook is on the final normalisation layer and fires at every
decoding step. For each slice we take $1000$ held-out prompts and greedily
generate $20$ tokens, and an offline classifier~\citep{detoxify} scores
toxicity. Baselines get the same strength dial as ours, written as
$h' = h + \alpha\,(g(h)-h)$ for a baseline map $g$, where $\alpha=1$ is the
published method. To measure content we need a label that is neither the
concept nor a task: a $20$-way $k$-means partition of frozen sentence
embeddings of the prompts. Topic accuracy is how recoverable that partition
is, and topicality is the prompt--continuation cosine similarity in the same
space. Topic alone predicts toxicity at $71.4\%$, which bounds how far any
method can erase without destroying content.

\paragraph{Autoregressive gender (Llama-2-7B).}
We append a fixed continuation cue to each raw \textsc{Bias in Bios}
biography and intervene at the same point, the LM-head input at the last token.
The autoencoder is trained on $50{,}000$ such states with gender labels only;
activation addition and MiMiC are fit on the same states. We greedily
generate $40$ tokens. Target preference is the fraction of first intervened
states that put more probability on target- than on source-gender tokens.
Perplexity is from rescoring the continuation with the untouched model. We
also count gender markers in the text and report profession accuracy for
content.

\paragraph{Frozen-probe gender (BERT, DeepSeek, LLaMA).}
On \textsc{Bias in Bios} we steer \textsc{female}$\to$\textsc{male} and
compare with the affine methods of \citep{singh2024representation}. Unlike
their protocol, the gender probe $q_C$ and the profession probe $q_Y$ are
trained once on the original representations and then frozen, so steering
must work in front of a system it cannot retrain. Conv.\ is how often $q_C$
predicts the target gender after steering, $q_{Y_{\mathrm{acc}}}$ is the
accuracy of the frozen profession probe, and $\mathrm{TPR}_{\mathrm{RMS}}$ is
defined in Appendix~\ref{app:steering_metrics}.

\paragraph{Reporting.}
All rows of a generation table come from one run on one GPU. Greedy decoding
is sensitive to floating-point kernel differences: $232$ of $1000$
continuations of the unmodified model changed between two GPU types, so we
never compare generation numbers across runs.

\subsection{Computational Complexity and Scalability}
\label{sec:complexity}
\begin{table}[!htbp]
\centering
\caption{Wall-clock time per training step on Nvidia H200 GPU
(mean $\pm$ std across 10 independent runs). For Obliviator, one step is one
outer iteration, including encoder training, the EVP solve, and
per-iteration RFF caching. For our method, one step is one training epoch.
Both methods use the same mini-batch size (16384) and encoder MLP.
Probing time is excluded.}
\label{tab:wallclock-per-step}
\small
\setlength{\tabcolsep}{4pt}
\begin{tabular}{l l c r r r r}
\toprule
\multirow{2}{*}{PLM} & \multirow{2}{*}{Dataset} & \multirow{2}{*}{Mode}
  & \multicolumn{3}{c}{Obliviator (s / outer iter)} & \textbf{Ours} \\
\cmidrule(lr){4-6} \cmidrule(lr){7-7}
 & & & Encoder & EVP & Total & (s / epoch) \\
\midrule

\multirow{6}{*}{BERT}
 & \multirow{2}{*}{\textsc{Bias in Bios}}
  & Unsup & $88.7 \pm 0.7$  & $7.77 \pm 0.15$ & $100.4 \pm 0.7$ & \multirow{2}{*}{$1.098 \pm 0.009$} \\
 & & Sup   & $56.9 \pm 0.6$  & $7.85 \pm 0.13$ & $66.7 \pm 0.7$  &  \\
\cmidrule(l){2-7}

 & \multirow{2}{*}{\textsc{DIAL-Mention}}
  & Unsup & $57.9 \pm 0.7$  & $7.40 \pm 0.21$ & $68.2 \pm 0.8$  & \multirow{2}{*}{$0.736 \pm 0.001$} \\
 & & Sup   & $33.3 \pm 0.2$  & $7.14 \pm 0.11$ & $42.0 \pm 0.2$  &  \\
\cmidrule(l){2-7}

 & \multirow{2}{*}{\textsc{Dial-Sentiment}}
  & Unsup & $58.2 \pm 0.5$  & $6.70 \pm 0.19$ & $67.8 \pm 0.7$  & \multirow{2}{*}{$0.734 \pm 0.004$} \\
 & & Sup   & $33.3 \pm 0.2$  & $6.75 \pm 0.23$ & $41.6 \pm 0.6$  &  \\

\midrule
\multirow{2}{*}{GPT-2}
 & \multirow{2}{*}{\textsc{Bias in Bios}}
  & Unsup & $90.0 \pm 1.8$  & $7.89 \pm 0.30$ & $101.7 \pm 2.0$ & \multirow{2}{*}{$1.103 \pm 0.010$} \\
 & & Sup   & $57.6 \pm 0.8$  & $7.90 \pm 0.19$ & $67.6 \pm 1.0$  &  \\

\midrule
\multirow{2}{*}{LLaMA}
 & \multirow{2}{*}{\textsc{Bias in Bios}}
  & Unsup & $130.2 \pm 2.3$ & $8.61 \pm 0.20$ & $144.0 \pm 2.2$ & \multirow{2}{*}{$1.650 \pm 0.007$} \\
 & & Sup   & $75.6 \pm 2.0$  & $8.13 \pm 0.28$ & $86.4 \pm 2.4$  &  \\

\midrule
\multirow{4}{*}{DeepSeek}
 & \multirow{2}{*}{\textsc{Bias in Bios}}
  & Unsup & $151.7 \pm 0.6$ & $8.90 \pm 0.13$ & $166.6 \pm 0.5$ & \multirow{2}{*}{$2.547 \pm 0.025$} \\
 & & Sup   & $75.0 \pm 0.4$  & $8.00 \pm 0.22$ & $85.6 \pm 0.6$  &  \\
\cmidrule(l){2-7}

 & \multirow{2}{*}{\textsc{Dial-Sentiment}}
  & Unsup & $92.9 \pm 0.5$  & $7.26 \pm 0.20$ & $103.9 \pm 0.6$ & \multirow{2}{*}{$1.693 \pm 0.022$} \\
 & & Sup   & $42.4 \pm 0.3$  & $6.92 \pm 0.27$ & $51.3 \pm 0.7$  &  \\

\bottomrule
\end{tabular}
\end{table}

\paragraph{Time analysis.} Table~\ref{tab:wallclock-per-step} reports wall-clock time per training
step. The main observation is that our per-step cost is substantially
smaller: across all configurations, one epoch of our method takes roughly
$0.7$--$2.6$ seconds, whereas one Obliviator outer iteration takes roughly
$42$--$167$ seconds.

To translate these numbers into full-run cost, we project Obliviator to
$K \in \{5,10,15\}$ outer iterations and our method to
$200$--$500$ epochs. These ranges match the training budgets used in our
experiments: Obliviator runs for up to $15$ outer iterations, while $500$
epochs is the largest number of epochs we ran for our method. Even under
the conservative comparison of our longest $500$-epoch runs against
Obliviator with $15$ outer iterations, our method is faster in every
configuration, for both supervised and unsupervised Obliviator. In the
unsupervised setting, which is the main post-hoc erasure setting, our
method is about $1.8$--$2.8\times$ faster. For example, on BERT
\textsc{Bias in Bios}, Obliviator-15 takes about $1506$ seconds, compared
with about $549$ seconds for our $500$-epoch run; on DeepSeek
\textsc{Bias in Bios}, the corresponding numbers are about $2499$ and
$1274$ seconds.

Overall, these runtime gains strengthen the empirical picture: our method
does not trade speed for quality, but improves the utility--erasure
trade-off while also reducing training time. Thus, the same simpler
objective that yields stronger erasure performance also makes the method
substantially more efficient in practice.

\begin{table}[H]
\centering
\caption{Peak GPU memory (GB) for \textsc{Bias in Bios}.}
\label{tab:memory}
\small
\begin{tabular}{lcc}
\toprule
Method & Training (GB) & EVP (GB) \\
\midrule
Obliviator & 8.18 & 6.78 \\
\textbf{Ours}  & \textbf{5.15} & --- \\
\bottomrule
\end{tabular}
\end{table}

\paragraph{Memory analysis.}
Both methods use the \emph{same} encoder MLP, optimizer, and mini-batch
size, so the encoder weights, gradients, and per-batch activations occupy
identical memory. Both also include an HSIC term against the sensitive
attribute $C$ (the small one-hot label vector being erased), which is
inexpensive for both. The remaining difference is structural to each
loss. Obliviator's objective (Eq.~8 in their paper) additionally includes HSIC terms
against fixed auxiliary random variables ($X$, $X^i$, and optionally $Y$)
used to preserve utility; their implementation precomputes these
as $N \times d_{\text{rff}}$ RFF feature matrices once per outer iteration
and caches them on GPU throughout encoder training. Ours preserves utility
through the autoencoder's MSE reconstruction loss, which requires no
auxiliary RFF features. As a result, Ours uses lower peak memory than
Obliviator's encoder phase while avoiding the EVP solve entirely.

\section{Additional Results}
\label{sec:add_results}
\subsection{Toxicity Steering}
\label{app:tox_results}

The setup is in Appendix~\ref{app:steering_setup}. We call the rule of $h'$ in
Section~\ref{sec:steering} \emph{residual} steering, and the plain conversion
$d(e(h),c')$ \emph{replacement}.

Replacement does not generate usable text at code size $512$: perplexity
rises to $2807$, against $3.44$ for the untouched model. The autoencoder
round-trip, which does no steering, degrades the same way, so the cost comes
from reconstruction error and not from the intervention. A
$2048$-dimensional code trained on the last four prompt positions does
generate (perplexity $10.9$) and lowers toxicity from $0.163$ to $0.050$,
against $0.145$ for its own round-trip.

\begin{table}[h]
\centering\small
\caption{Toxicity recoverable by fresh probes, as deployed and after undoing
each map.}
\label{tab:tox_adv_app}
\begin{tabular}{lccc}
\toprule
Intervention & Invertible & As deployed & Map undone\\
\midrule
Mean matching       & yes & 99.2 & 99.2\\
Activation addition & yes & 100.0 & 100.0\\
MiMiC               & yes & \textbf{71.3} & 84.9\\
\midrule
Ours, replacement   & no  & 77.1 & \textbf{77.1}\\
\bottomrule
\end{tabular}
\end{table}

An adversary who knows a published map can invert it whenever it is
invertible. Table~\ref{tab:tox_adv_app} trains fresh probes on the steered
states as deployed, and again after undoing each map with its analytic
inverse. Every affine baseline is invertible, so undoing it restores the
concept: MiMiC reads lowest as deployed ($71.3\%$), but $84.9\%$ once undone.
Replacement passes through a $512$-dimensional code, has no inverse, and
stays at $77.1\%$. Untouched states read $86.8\%$, chance is $50\%$, and topic
alone predicts toxicity at $71.4\%$.

\begin{table}[h]
\centering\small
\caption{Frozen-probe steering on Qwen states.}
\label{tab:tox_frozen_app}
\begin{tabular}{lccc}
\toprule
Intervention & Conv.\ \% $\uparrow$ & $q_{Y_{\mathrm{acc}}}$ \% $\uparrow$ & $\mathrm{TPR}_{\mathrm{RMS}}$ $\downarrow$\\
\midrule
No intervention          & 14.60 & 51.97 & 0.195\\
\midrule
Mean matching            & 79.52 & 47.92 & 0.121\\
Activation addition      & 99.41 & 40.13 & 0.194\\
MiMiC                    & 93.72 & 41.72 & 0.160\\
\midrule
Ours, residual $\beta=1$ & 71.27 & \textbf{47.60} & 0.111\\
Ours, residual $\beta=2$ & 97.14 & 39.56 & 0.185\\
Ours, replacement        & 92.88 & 44.80 & \textbf{0.080}\\
Round-trip (control)     & 9.18 & 49.12 & 0.205\\
\bottomrule
\end{tabular}
\end{table}

\begin{table}[h]
\centering\footnotesize
\caption{Selected greedy continuations, untouched model vs.\ ours (residual,
$\beta=2$).}
\label{tab:tox_examples}
\scalebox{0.85}{%
\begin{tabular}{p{0.26\linewidth}p{0.33\linewidth}p{0.33\linewidth}}
\toprule
Prompt & Untouched & Ours\\
\midrule
Dueling Million Moron Marches In DC Get Like A Hundred Idiots To Show Up You know things & \ldots{}are bad when the two biggest moron marches in DC get like a hundred idiots to show up. & \ldots{}are bad when the two largest marches in Washington DC in 2017 both involved people marching\\
\midrule
But even if he only leaked Belus because he thinks it's ``garbage'' & \ldots{}and ``a waste of time,'' he's still a traitor. He's still a traitor & \ldots{}he still has a responsibility to the public to ensure that information that could be useful to the public is\\
\midrule
``I would like to talk to you about your comments about my being a racist, & \ldots{}and I would like to talk to you about the fact that you are a racist.'' & \ldots{}which I find to be a very serious charge. I have been a member of the National Association for\\
\midrule
I was glad to see someone who, frankly, isn't an idiot -- and doesn't act & \ldots{}like one -- on the show. I was also glad to see someone who isn't a complete idiot & \ldots{}like one -- on the show. The only reason I watch ``The Office'' is because I like\\
\bottomrule
\end{tabular}}
\end{table}

In Table~\ref{tab:tox_frozen_app} a toxicity probe $q_C$ and a topic probe
$q_Y$ are trained once on untouched states and frozen, and each intervention
is inserted in front of them. Conv.\ is how often $q_C$ calls a converted
toxic state non-toxic, $q_{Y_{\mathrm{acc}}}$ is topic accuracy, and
$\mathrm{TPR}_{\mathrm{RMS}}$ is the topic-wise gap between the original
toxicity groups. Replacement converts as often as MiMiC ($92.9\%$ vs.\
$93.7\%$), keeps more topic accuracy ($44.8$ vs.\ $41.7$) and halves the
fairness gap ($0.080$ vs.\ $0.160$).

Table~\ref{tab:tox_examples} shows continuations where the untouched model
turns the prompt into an insult and ours continues on topic without it.
Toxicity drops from at least $0.65$ to below $0.01$ in each case, at
perplexity $4$--$6$.

\subsection{Gender Steering on Llama-2-7B}
\label{sec:gender_autoregressive}

\paragraph{Direct LM-head steering.}
\begin{table}[h]
\centering\small
\caption{Direct LM-head steering on Llama-2-7B.}
\label{tab:gender_direct}
\begin{tabular}{lccc}
\toprule
Intervention & Target pref.\,$\uparrow$ & Original gender $\downarrow$ & Profession $\uparrow$\\
\midrule
Untouched           & --    & 78.1 & 56.2\\
Mean matching       & 81.4  & 93.8 & 56.2\\
Activation addition & 92.6  & 99.3 & 56.2\\
MiMiC               & 90.2  & 89.1 & 54.4\\
Ours                & \textbf{99.2} & \textbf{63.3} & 49.5\\
\bottomrule
\end{tabular}
\end{table}
We replace $h$ by $d(e(h),c')$ once and read the frozen head's next-token
distribution. Target preference is the fraction of states where target-gender
tokens get more probability than source-gender tokens; the other two columns
are what fresh probes recover from the converted state
(Table~\ref{tab:gender_direct}). Ours moves the head most often ($99.2\%$) and
is the only method that removes most of the original gender ($63.3\%$ against
$78.1\%$ untouched; the baselines leave $89$--$99\%$), at a cost of $6.7$
points of profession accuracy.

\begin{table}[!h]
\centering
\small
\setlength{\tabcolsep}{4pt}
\caption{Autoregressive gender steering on Llama-2-7B. PPL is computed under
the untouched model; lower is better.}
\label{tab:gender_autoregressive}
\begin{tabular}{lcccc}
\toprule
\textbf{Intervention}
& \textbf{Target pref. $\uparrow$}
& \textbf{PPL $\downarrow$}
& \textbf{Target/source}
& \textbf{Profession $\uparrow$} \\
\midrule
Untouched
& -- & 2.28 & 27 / 13 & 56.2 \\

$d(e(h),c')$
& 99.2\% & 8.97 & 8 / 0 & 49.5 \\

Ours, $\beta=1.0$
& 89.5\% & 2.46 & 39 / 0 & 56.1 \\

Ours, $\beta=1.5$
& 94.8\% & 2.65 & 51 / 0 & 56.0 \\

Activation addition
& 92.6\% & 2.36 & 32 / 3 & 56.2 \\

MiMiC
& 90.2\% & 3.53 & 109 / 0 & 54.4 \\
\bottomrule
\end{tabular}
\end{table}

\paragraph{Autoregressive generation.}

In Table~\ref{tab:gender_autoregressive}, direct replacement produces the
strongest initial target preference, confirming
that $d(e(h),c')$ directly converts the representation toward the requested
concept. When repeated throughout generation, however, its reconstruction
error accumulates, substantially increasing perplexity and reducing profession
preservation. Residualization avoids this degradation while retaining strong
steering: at $\beta=1.0$, all observed source-gender markers are removed while
perplexity remains close to the untouched model and profession accuracy is
essentially unchanged. Increasing to $\beta=1.5$ further strengthens target
steering at only a modest fluency cost. This is why the toxicity experiment uses the residual form.

\subsection{Frozen-Probe Gender Steering}
\label{sec:steering_bios}

\begin{table}[h]
\centering
\small
\setlength{\tabcolsep}{3.5pt}
\caption{Frozen-probe gender steering across PLMs.}
\label{tab:steering_bios}
\resizebox{0.60\linewidth}{!}{
\begin{tabular}{llccc}
\toprule
Model & Intervention & Conv.\% $\uparrow$
& $q_{Y_\mathrm{acc}}$\% $\uparrow$
& $\mathrm{TPR}_{\mathrm{RMS}}\downarrow$ \\
\midrule
\multirow{3}{*}{BERT}
& Mean matching      & $\mathbf{99.97}$\% & $78.18$\%          & 0.109 \\
& Mean+Cov. Matching & $99.86$\%          & $73.95$\%          & 0.154 \\
& Ours               & $99.92$\%          & $\mathbf{78.20}$\% & $\mathbf{0.105}$ \\
\midrule
\multirow{3}{*}{DeepSeek}
& Mean matching      & $99.94$\%          & $\mathbf{84.21}$\% & 0.115 \\
& Mean+Cov. Matching & $99.84$\%          & $78.13$\%          & 0.178 \\
& Ours               & $\mathbf{99.98}$\% & $83.65$\%          & $\mathbf{0.110}$ \\
\midrule
\multirow{3}{*}{LLaMA}
& Mean matching      & $99.56$\%          & $\mathbf{82.28}$\% & 0.100 \\
& Mean+Cov. Matching & $99.85$\%          & $77.12$\%          & 0.141 \\
& Ours               & $\mathbf{99.99}$\% & $81.83$\%          & $\mathbf{0.097}$ \\
\bottomrule
\end{tabular}
}
\end{table}

Table~\ref{tab:steering_bios} reports three quantities. Conv. is the rate at
which the frozen gender classifier $q_C$ predicts the target gender,
$q_{Y_\mathrm{acc}}$ is frozen profession accuracy, and
$\mathrm{TPR}_{\mathrm{RMS}}$ is the profession fairness gap defined in
Appendix~\ref{app:steering_metrics}.

Across PLMs, our method achieves near-perfect conversion while preserving high
profession accuracy and consistently attaining the lowest
$\mathrm{TPR}_{\mathrm{RMS}}$. Mean+covariance matching, in contrast, often
reduces compatibility with the frozen downstream classifier, decreasing
profession accuracy and increasing the fairness gap. These results show that
our decoder-based intervention reliably changes the target concept while
preserving the representation structure required by an already-trained
downstream system.

\subsection{Additional Steering Results}
\label{sec:add_steer}
\begin{table}[h]
\centering
\newcommand{\steeringtablescale}{0.80} 
\caption{%
  AE-only steering quality evaluation. Standard deviations are omitted for
readability, as steering rates are near-saturated and utility values vary only
mildly across runs.
}
\label{tab:section_d}
\scalebox{\steeringtablescale}{%
\begin{tabular}{lcccc}
\toprule
 & & \textbf{BERT} & \textbf{DeepSeek} & \textbf{LLaMA} \\
\midrule
\multicolumn{5}{l}{\textit{Reconstruction quality}} \\
 & Round-trip MSE & 0.0058 & 0.0014 & 0.0643 \\
\midrule
\multicolumn{5}{l}{\textit{Per-class steering success ($q_C$ predicts $c'$ after steering)}} \\
 & F$\to$M ($c'\!=\!0$): no-steer   & 0.70\%  & 0.70\%  & 0.83\%  \\
 & F$\to$M ($c'\!=\!0$): round-trip & 0.16\%  & 0.00\%  & 0.00\%  \\
 & F$\to$M ($c'\!=\!0$): converted  & \textbf{99.92\%} & \textbf{99.98\%} & \textbf{99.99\%} \\
\cmidrule{2-5}
 & M$\to$F ($c'\!=\!1$): no-steer   & 0.46\%  & 0.35\%  & 0.39\%  \\
 & M$\to$F ($c'\!=\!1$): round-trip & 0.02\%  & 0.01\%  & 0.01\%  \\
 & M$\to$F ($c'\!=\!1$): converted  & \textbf{99.45\%} & \textbf{100.00\%} & \textbf{99.99\%} \\
\midrule
\multicolumn{5}{l}{\textit{Utility preservation (re-trained probe on AE outputs; raw data in parens)}} \\
 & Round-trip $D(E(x),c_{\mathrm{orig}})$ & 80.06\% (80.53\%) & 84.53\% (84.81\%) & 82.93\% (83.35\%) \\
 & All-converted to $c\!=\!0$             & 79.96\%           & 84.51\%           & 82.78\%           \\
 & All-converted to $c\!=\!1$             & 79.90\%           & 84.48\%           & 82.75\%           \\
\bottomrule
\end{tabular}
}
\end{table}

We evaluate the AE in isolation, independent of the fairness metrics in Table~\ref{tab:steering_bios}.
All measurements use a frozen gender probe $q_C$ (an Obliviator MLP trained once on raw data,
accuracy $\geq\!99.4\%$ across PLMs) and a frozen profession probe $q_Y$ trained on raw data.

\paragraph{Metrics.}
\emph{Round-trip MSE} measures reconstruction fidelity: $\|D(E(x),\,c_{\mathrm{orig}})-x\|^2$,
where the encoder output is decoded with the \emph{original} concept label.
Lower values indicate a more faithful autoencoder.
\emph{Per-class steering success} reports how often $q_C$ classifies the AE output
$D(E(x),c')$ as the \emph{target} gender $c'$.
The \emph{no-steer} row is the baseline rate before any intervention (near zero by
construction of $q_C$); the \emph{round-trip} row verifies that decoding with the original
concept does not accidentally flip the gender label; the \emph{converted} row is the key
metric.
\emph{Utility preservation} re-trains a fresh profession probe from scratch on AE outputs
and reports its accuracy.
The parenthetical values are the raw $q_Y$ on unmodified representations, serving as an upper bound.

\paragraph{Results.}
Steering is near-perfect and symmetric across all PLMs ($\geq\!99.45\%$ in both
directions), confirming that the concept variable $c$ genuinely controls gender in the
learned latent space.
The round-trip column stays below $0.2\%$, ruling out the degenerate solution of
ignoring $c$ and always outputting a fixed gender.
Utility is preserved within $0.6\%$ of raw $q_Y$ in all cases, and is nearly identical
whether decoding with the original concept or converting all samples to either gender --
indicating that $c$ is cleanly orthogonal to profession content.
Round-trip MSE values are small across all PLMs (below $0.065$), indicating that the autoencoder faithfully reconstructs its input when the concept label is held fixed, a necessary condition for meaningful concept-controlled generation.

Taken together, these results confirm that the AE successfully isolates gender in the concept variable $c$: it converts representations with near-perfect accuracy, reconstructs faithfully when $c$ is held fixed, and preserves profession content within $0.6$\% of the raw baseline -- demonstrating clean disentanglement without sacrificing task utility.
\subsection{Additional Probe Families}
\label{app:more_probes}
\begin{table}[h]
\centering\small
\caption{Additional probe families: concept accuracy $\downarrow$ / task accuracy $\uparrow$.}
\label{tab:more_probes}
\begin{tabular}{lccc}
\toprule
Probe & \textsc{Bias in Bios} & \textsc{Dial-Mention} & \textsc{Dial-Sentiment}\\
\midrule
4-layer MLP & 55.63 / 73.31 & \textbf{52.90} / 69.10 & 64.10 / \textbf{67.20}\\
RBF-SVM     & \textbf{54.30} / \textbf{75.40} & 56.00 / \textbf{70.70} & 63.80 / 65.40\\
kNN         & 56.80 / 74.80 & 57.30 / 70.30 & \textbf{62.40} / 62.10\\
\bottomrule
\end{tabular}
\end{table}
Our main results use a 5-layer MLP, selected as the strongest probe in our
probe search. As an additional robustness check, we also evaluate a separate
4-layer MLP, together with an RBF-kernel SVM and kNN. Since a single probe family
may nevertheless miss leakage detectable by another, we attack the erased
representations with all three alternatives. The 4-layer MLP is evaluated over
3 random restarts, the RBF-SVM over 7 hyperparameter settings, and kNN over
2 settings, with each probe trained from scratch. Each cell of
Table~\ref{tab:more_probes} reports concept accuracy followed by task
accuracy on the same representation. The concept majority baselines are
53.7\%, 50.1\%, and 63.1\%.

On \textsc{Bias in Bios}, all three probe families recover gender within
$0.6$--$3.1$ points of the majority baseline, while profession accuracy remains
between $73.3\%$ and $75.4\%$. On \textsc{Dial-Sentiment}, all probes remain
within $1.0$ point of the majority baseline. \textsc{Dial-Mention} shows the
largest spread: the MLP reaches $52.9\%$, while SVM and kNN reach $56.0\%$ and
$57.3\%$, respectively. Even there, concept accuracy remains within about
$7$ points of chance and mention accuracy stays above $69\%$. The main
conclusions therefore do not depend on the MLP probe family.

\subsection{Permutation Test}
\label{app:permutation}
\begin{table}[h]
\centering\small
\caption{Rejections of independence out of $20$ disjoint subsamples.}
\label{tab:permutation}
\begin{tabular}{llccc}
\toprule
Dataset & Representation & $n=1{,}000$ & $n=2{,}000$ & $n=5{,}000$\\
\midrule
\multirow{2}{*}{\textsc{Dial-Mention}}   & Frozen BERT & 20/20 & 20/20 & 20/20\\
                                         & Ours        & \textbf{0/20} & \textbf{4/20} & \textbf{13/20}\\
\midrule
\multirow{2}{*}{\textsc{Dial-Sentiment}} & Frozen BERT & 20/20 & 20/20 & 20/20\\
                                         & Ours        & \textbf{0/20} & \textbf{4/20} & \textbf{10/20}\\
\bottomrule
\end{tabular}
\end{table}
Probes can miss leakage or overfit noise, so we also test independence
directly. On held-out data we compute the exact-Gram HSIC between the code
and the concept (RBF kernel on $Z$, one-hot kernel on $C$) and run a
permutation test with $500$ permutations at level $0.05$, on $20$ disjoint
subsamples per sample size $n$ (Table~\ref{tab:permutation}). For frozen BERT,
independence is rejected every time, already at $n=1{,}000$. For ours it is
never rejected at $n=1{,}000$ and only partly at $n=5{,}000$: the remaining
dependence takes thousands of samples to detect. Held-out HSIC falls by at
least $84\times$ on \textsc{Dial-Mention} and $67\times$ on
\textsc{Dial-Sentiment} relative to frozen BERT, comparing our upper $90\%$
bootstrap endpoint with frozen BERT's lower one.

\subsection{Fairness}
\label{sec:fariness}
We further evaluate whether erasure improves downstream fairness on
\textsc{Dial-Sentiment}. Following Obliviator, we plot sentiment accuracy against
two fairness metrics: demographic parity (DP) and the RMS true-positive-rate gap,
denoted Gap$_{\mathrm{rms}}$. Both metrics are defined in
Appendix~\ref{app:fairness_metrics}. Lower values on the $x$-axis indicate fairer
predictions, while higher values on the $y$-axis indicate better sentiment
utility.

\begin{figure}[h]
    \centering
    \includegraphics[width=0.75\textwidth]{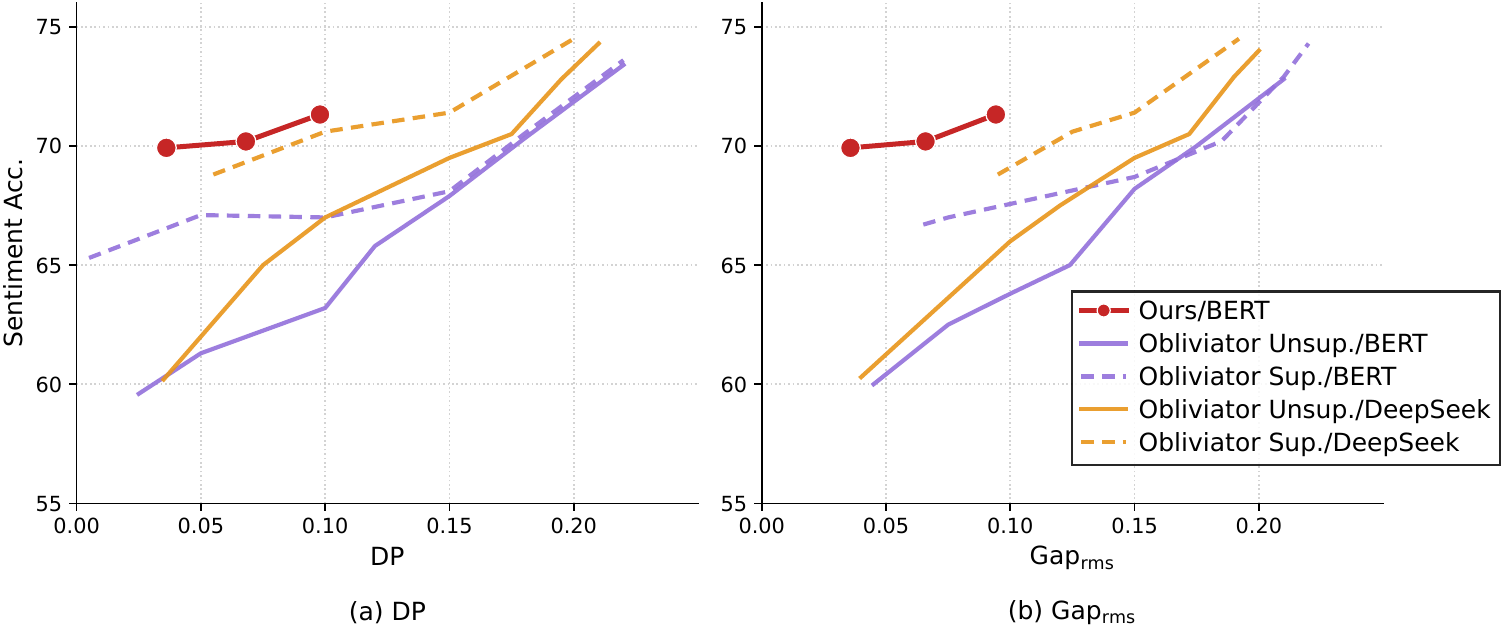}
    \caption{
    Fairness--utility trade-off on \textsc{Dial-Sentiment}.}
    \label{fig:dial_sentiment_fairness}
\end{figure}

Figure~\ref{fig:dial_sentiment_fairness} shows that our method achieves a
stronger fairness--utility trade-off on \textsc{Dial-Sentiment}. At the same
levels of fairness, our BERT representations preserve substantially more
sentiment information than Obliviator. For example, at DP around $0.03$, our
method reaches roughly $69\%$ sentiment accuracy, whereas unsupervised
Obliviator on BERT is around $60\%$ and supervised Obliviator on BERT is around
$67\%$. Even compared to Obliviator on DeepSeek, which starts from a more
favorable representation space, our BERT result is stronger at this fairness
level: unsupervised Obliviator on DeepSeek is around $61\%$, while supervised
Obliviator on DeepSeek does not reach this low-DP regime.

The same pattern appears for Gap$_{\mathrm{rms}}$. At Gap$_{\mathrm{rms}}$
around $0.04$, our method reaches roughly $70\%$ sentiment accuracy, whereas
unsupervised Obliviator on both BERT and DeepSeek barely reaches $60\%$, and the
supervised Obliviator curves do not reach this fairness level. At a slightly
larger Gap$_{\mathrm{rms}}$ of about $0.095$, our method reaches roughly $71\%$
sentiment accuracy, still above all Obliviator curves at nearby fairness levels,
including DeepSeek. Overall, these results show that our method achieves strong
downstream fairness while preserving utility, despite using no sentiment labels
during erasure.

\subsection{Hyperparameter Sensitivity}
\label{sec:hyperparameter_sens}
\begin{figure}[H]
    \centering
    \includegraphics[width=1.0\textwidth]{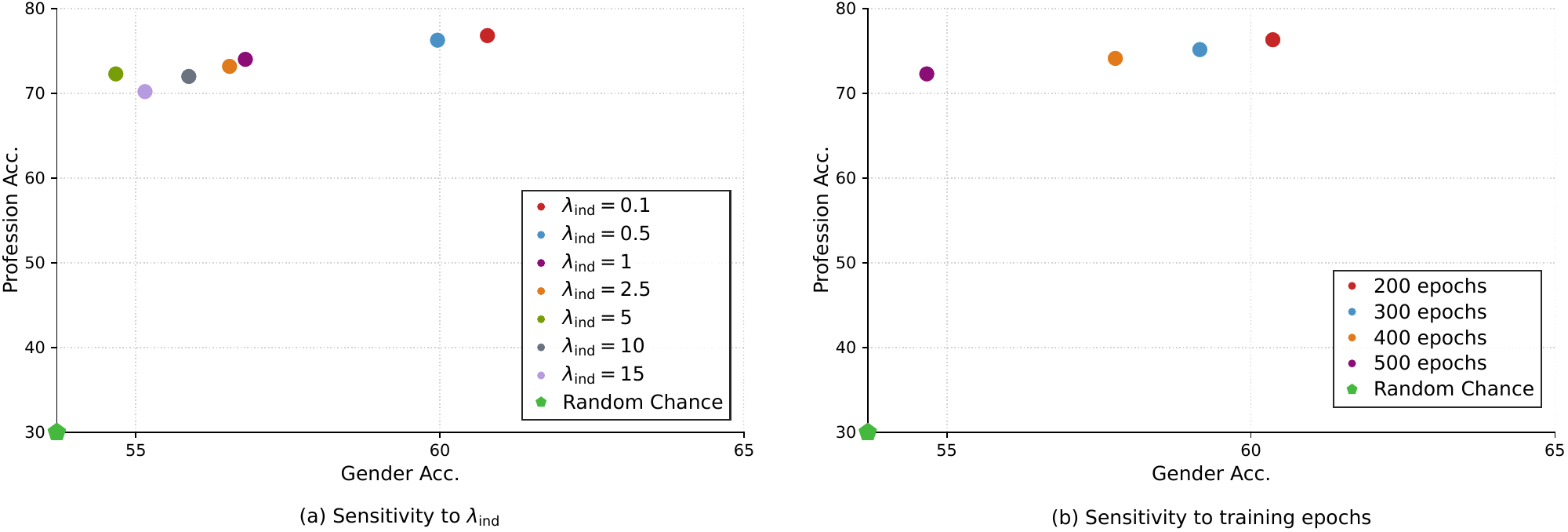}
    \caption{
    Hyperparameter sensitivity on \textsc{Bias in Bios} with BERT representations.
    We vary the independence weight $\lambda_{\mathrm{ind}}$ in Eq.~\ref{eq:objective}
    and the number of training epochs.
    }
    \label{fig:hyperparameter_sensitivity_bert_bios}
\end{figure}

Figure~\ref{fig:hyperparameter_sensitivity_bert_bios} shows that our method is
robust and not highly sensitive to the exact choice of
$\lambda_{\mathrm{ind}}$ over a broad range. Small values ($0.1$--$0.5$)
preserve about $76$--$77\%$ profession accuracy, while increasing
$\lambda_{\mathrm{ind}}$ progressively strengthens erasure. At
$\lambda_{\mathrm{ind}}=5$, gender accuracy reaches $54.67\%$, close to the
$53.7\%$ random-chance level, while profession accuracy remains $72.29\%$.
Larger values provide little additional erasure benefit, indicating a broad
and stable operating regime rather than dependence on a narrowly tuned value.

Training duration shows a similarly smooth trade-off. From $200$ to $500$
epochs, gender accuracy decreases steadily from $60.36\%$ to $54.67\%$, while
profession accuracy decreases from $76.32\%$ to $72.29\%$. We use $500$ epochs
in the main experiments because it provides the strongest erasure among the
tested budgets while retaining substantial downstream utility.

\section{Metrics}
\label{app:metrics}

\subsection{Fairness Metrics}
\label{app:fairness_metrics}

For the downstream fairness evaluation on \textsc{Dial-Sentiment}, we report
demographic parity (DP) and the RMS true-positive-rate gap, denoted
Gap$_{\mathrm{rms}}$. DP measures whether the predicted sentiment distribution
changes across race groups:
\[
\mathrm{DP}
=
\mathbb{E}_{y \in \mathcal{Y}}
\left[
\left|
\Pr(\widehat{Y}=y \mid C=0)
-
\Pr(\widehat{Y}=y \mid C=1)
\right|
\right].
\]
Gap$_{\mathrm{rms}}$ measures class-conditioned disparity in true positive
rates:
\[
\mathrm{Gap}_{\mathrm{rms}}
=
\left(
\mathbb{E}_{y \in \mathcal{Y}}
\left[
\left(
\Pr(\widehat{Y}=y \mid Y=y, C=0)
-
\Pr(\widehat{Y}=y \mid Y=y, C=1)
\right)^2
\right]
\right)^{1/2}.
\]
Lower values indicate fairer predictions.

\subsection{Steering Metrics}
\label{app:steering_metrics}

For steering, we follow the fairness metric used in prior representation-steering
work and report the RMS true-positive-rate gap. For each profession label \(y\),
let \(\widehat{Y}\) be the prediction of the frozen profession classifier and
let \(C\) denote the gender group. The TPR gap for class \(y\) is
\[
\mathrm{TPR\text{-}Gap}(y)
=
\left|
\Pr(\widehat{Y}=y \mid Y=y, C=0)
-
\Pr(\widehat{Y}=y \mid Y=y, C=1)
\right|.
\]
We aggregate over profession labels using
\[
\mathrm{TPR}_{\mathrm{RMS}}
=
\sqrt{
\frac{1}{|\mathcal{Y}|}
\sum_{y \in \mathcal{Y}}
\mathrm{TPR\text{-}Gap}(y)^2
}.
\]
Lower values indicate smaller class-conditioned disparity across gender groups.

\end{document}